\documentclass{article}

\PassOptionsToPackage{numbers, compress}{natbib}

\usepackage[main, final]{neurips_2026}

\usepackage[utf8]{inputenc} 
\usepackage[T1]{fontenc}    
\usepackage{hyperref}       
\usepackage{url}            
\usepackage{booktabs}       
\usepackage{amsfonts}       
\usepackage{nicefrac}       
\usepackage{microtype}      
\usepackage{xcolor}         
\usepackage{subcaption}      
\usepackage{graphicx}
\usepackage{multirow}       
\usepackage{wrapfig}        
\usepackage{smile}
\usepackage[most]{tcolorbox}

\usepackage{enumitem} 

\title{DiffATS: Diffusion in Aligned Tensor Space}

\author{%
  \textbf{Jinhua Lyu, Tianmin Yu, Brian Kim, Lizhuo Zhou,
  Chanwook Park, Naichen Shi}$^{\dagger}$\\
  Northwestern University\\
  $^{\dagger}$Corresponding author:
  \texttt{naichen.shi@northwestern.edu}
}

\begin{document}

\maketitle

\begin{abstract}
Direct diffusion modeling of high-resolution spatiotemporal fields is computationally challenging. 
Parameter-efficient primitives address this by representing high-dimensional data with a compact set of parameters.
In this paper, we construct data-dependent tensor primitives 
without pretrained compression autoencoders. 
Our construction starts from Tucker decomposition, 
which captures low-rank multilinear structure through a core tensor and mode-wise factors.
However, Tucker factors are non-unique:
the same tensor can be represented by different rotated factors, 
which complicates generative modeling. 
We address this issue with orthogonal Procrustes (OP) alignment.
Specifically, we select medoid anchor matrices 
from the data and align the factor matrices to resolve the gauge ambiguity. 
This yields matrix Grassmannian primitives and tensor Grassmannian primitives 
that are compact, data-adaptive, and directly decodable by explicit multilinear reconstruction.
Theoretically, we prove that the proposed primitive maps are homeomorphisms
between low-rank tensors and their corresponding primitive
spaces, certifying that the representations are non-degenerate and
topologically faithful. 
Building on these primitives, we propose
\emph{Diffusion in Aligned Tensor Space} (DiffATS), a generative framework that
trains diffusion models directly on aligned tensor primitives. 
Across images,
videos, and PDE solutions, DiffATS achieves strong unconditional and
conditional generation performance while compressing original data by \(3.9\times\) to \(210\times\), 
without relying on any pretrained deep compression autoencoders.
\end{abstract}

\begin{figure}[h]
    \centering
    \includegraphics[width=1.00\textwidth]{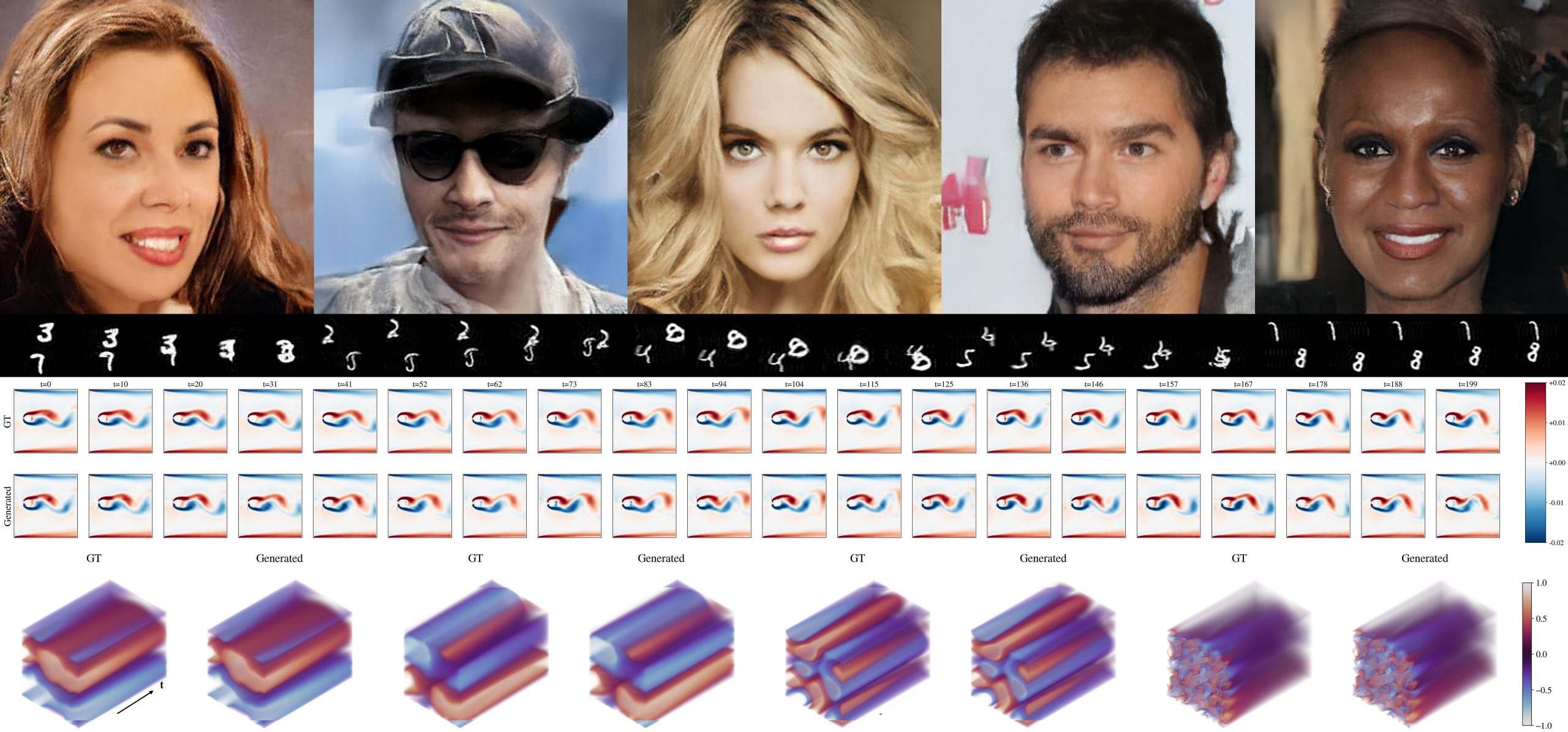}
    \caption{\textbf{Generation results of DiffATS on CelebA-HQ 1024, Moving MNIST, 2-d K\'arm\'an vortex street and 2-d Burgers' equation.}
    ``GT'' denotes the ground truth.}
    \label{fig:all_generation}
\end{figure}

\section{Introduction}

Diffusion models (DMs)~\cite{sohl2015deep, song2019generative, ho2020denoising, song2020score, ddim}
have become central tools for learning and sampling from complex unknown distributions,
with success in generating images~\cite{ho2020denoising,ddim},
audio~\cite{kong2020diffwave,liu2023audioldm}, videos~\cite{ho2022video,liu2024sora}, protein structures~\cite{abramson2024accurate}, 
inorganic materials~\cite{zeni2025generative}, and flow fields~\cite{shu2023physics}.
However, standard DMs often require huge computational resources to model high-resolution spatiotemporal data that are naturally represented as large tensors, such as videos, weather observations,
fluid states, and parametric PDE solutions. Direct diffusion in the original tensor space is computationally demanding as training and sampling require repeated function
evaluations in the full data dimension~\cite{rombach2022high}. This cost is particularly severe
for scientific tensors. For instance, GenCast~\cite{price2023gencast}, a diffusion model for
global weather forecasting, trains on tensors of size \(84\times720\times1440\) and requires
3.5 days on 32 TPUs.

To mitigate the computational cost, a common strategy is to perform diffusion 
in a lower-dimensional representation space. 
Latent diffusion models (LDMs)~\cite{rombach2022high} first train a regularized variational
autoencoder~\cite{kingma2013auto,rezende2014stochastic} to compress data into latent embeddings,
then learn a diffusion model in the latent space. Many subsequent image synthesis methods
build on this framework, such as transformer-based architectures~\cite{peebles2023scalable}
and high-resolution text-to-image models~\cite{gu2022vector, podell2023sdxl}. 
The same approach extends to video generation~\cite{he2022latent,blattmann2023align,ma2024latte},
with the latent
space coming from a pretrained or task-specific video autoencoder. LDMs have become
the de facto standard for natural images and videos, where large-scale pretrained autoencoders are
available.

Scientific data poses a different challenge.  
For the generative modeling of scientific problems, such as weather forecasting~\cite{price2023gencast} and 
turbulent flow generation~\cite{lienen2023zero},
off-the-shelf autoencoders pretrained on large-scale datasets are often not available. 
Training task-specific autoencoders is expensive: the cost scales with both network size
and data dimensionality. Training also requires sophisticated 
optimization techniques~\cite{van2017neural, esser2021taming,gu2022vector}.
This motivates autoencoder-free primitives that compress data through explicit transforms.
Recent examples include primitives based on discrete cosine transforms~\cite{ning2024dctdiff}
and tensor factorizations~\cite{chen2025generating,guo2026tucker}. These methods reduce
dimension efficiently, yet they rely on predefined cosine or tensor factor bases,
which can limit their ability to approximate the raw tensors. 
In light of these challenges, this paper asks the following question:

\begin{center}
    \textit{ Can we construct data-dependent parameter-efficient primitives for 
    high-resolution spatiotemporal data by exploiting their inherent low-rank structure 
    to improve generative modeling?}
\end{center}

The starting point is the Tucker decomposition~\cite{tucker1966some}, which represents a tensor
\(\cX\in\RR^{n_1\times\cdots\times n_d}\) by a core tensor and $d$ factor matrices.
This decomposition gives a natural data-adaptive and low-rank parametrization. 
However, as evidenced in Fig.~\ref{fig:toy_model}, 
training DMs on this parametrization yields inferior performance. 
We attribute this to the degeneracy of the standard Tucker decomposition: 
infinitely many combinations of core tensors and factor matrices yield the same tensor $\cX$.
This arbitrariness breaks the continuity of the primitive map, 
and complicates downstream modeling.

Our primary contribution lies in resolving this degeneracy problem through the orthogonal Procrustes (OP) alignment of factor matrices.
Under proper conditions, OP selects a unique representative from each Grassmannian equivalence class.
Building on this technique, we define two families of primitives:
\emph{matrix Grassmannian primitives} (MGPs) for low-rank matrices and
\emph{tensor Grassmannian primitives} (TGPs) for low-multilinear-rank tensors.
They preserve the parameter efficiency of Tucker representations 
while removing the
rotational ambiguity that destabilizes DM training.
Our theoretical analysis further shows that the primitive map \(\Phi\) is a homeomorphism
under an anchor-overlap condition, which guarantees that
the primitive space preserves the local topological structure of the original data space.

By combining these primitives with diffusion modeling,
we propose \emph{Diffusion in Aligned Tensor Space} (DiffATS). 
The pipeline is shown in Fig.~\ref{fig:pipeline}.
During training, DiffATS maps each high-dimensional tensor (or matrix) to
its TGP (or MGP) and trains DMs in the aligned primitive space. During inference,
the generated primitive is mapped back to the original data space through the
multilinear reconstruction map.

\begin{figure}[t]
    \centering
    \includegraphics[width=1.00\textwidth]{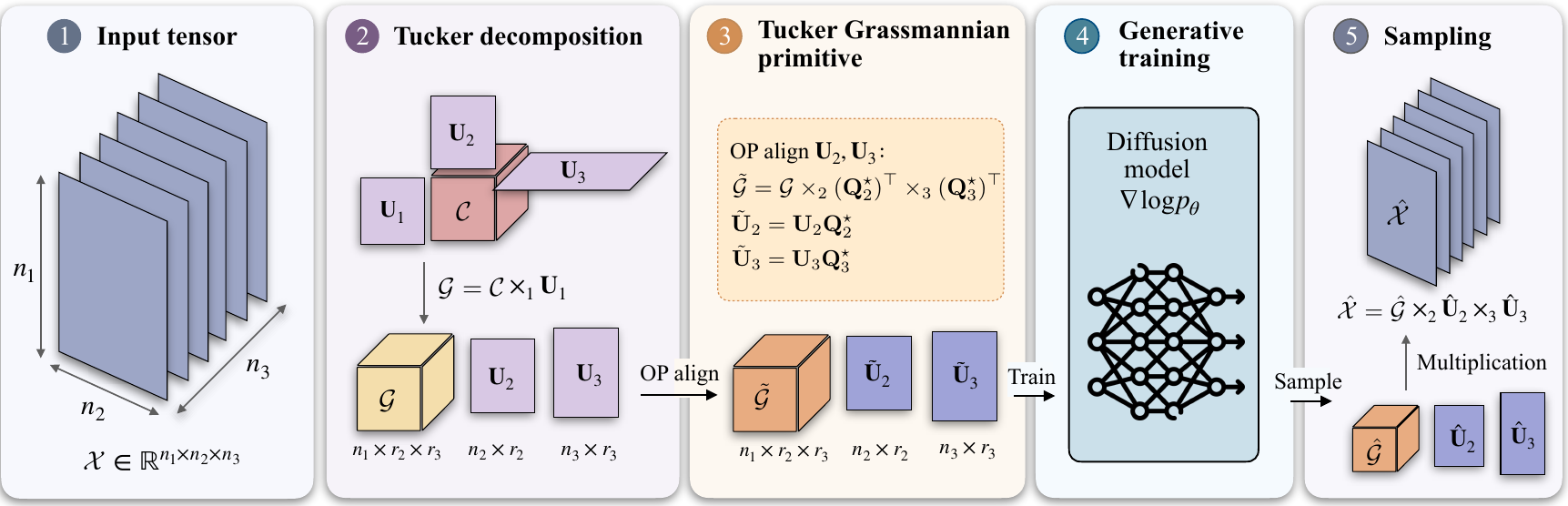}
    \caption{\textbf{Pipeline overview of DiffATS.}}
    \label{fig:pipeline}
\end{figure}

We conduct experiments on images, videos, and PDE solutions, 
with representative samples shown in Fig.~\ref{fig:all_generation}. 
Across all settings, DiffATS outperforms DCTdiff~\cite{ning2024dctdiff}, 
SDIFT~\cite{chen2025generating}, FNO~\cite{li2020fourier}, 
and average pooling (AvgPool) in generation quality, 
despite using compression ratios that match or exceed those of the baselines.


Our main contributions are summarized as follows:
\begin{itemize}[noitemsep]
    \item We introduce MGPs and TGPs as
    data-adaptive, parameter-efficient representations for low-rank matrices and tensors.
    These primitives retain the expressive power of Tucker decompositions while resolving
    the rotational ambiguity through anchored OP alignment.
    \item Theoretically, we prove the continuity and bijectivity of the primitive map $\Phi$ under overlapping conditions.
    To our knowledge, this is the first low-rank primitive family with
    a formal homeomorphism guarantee.
    \item We propose DiffATS, an end-to-end autoencoder-free diffusion framework that trains and samples
    directly in the MGP and TGP spaces, and reconstructs generated samples via the multilinear reconstruction map.
    \item We conduct extensive experiments on images, videos, and PDE solutions, and show that DiffATS outperforms
    all baselines on generation quality given a matching or higher compression ratio.
\end{itemize}





\section{Related Works}
\noindent\textbf{Autoencoder-free Diffusion Models.}
Among autoencoder-free approaches, DCTdiff~\cite{ning2024dctdiff} maps images to the DCT frequency domain and 
discards the highest-frequency coefficients as perceptually negligible, training a DM directly on the 
retained ones. SDIFT~\cite{chen2025generating} and Tucker-based diffusion~\cite{guo2026tucker} 
apply Tucker (or functional Tucker) decomposition with factor matrices shared across the dataset, 
training DMs only on the per-sample core tensor. In contrast, our method uses per-sample 
Tucker factors aligned to a medoid anchor via OP, and trains 
the DM jointly on the core tensor and the aligned factor matrices.

\noindent\textbf{Symmetry Structures in Diffusion Models.}
The factor matrices in tensor decompositions such as SVD~\cite{eckart1936approximation} and Tucker~\cite{tucker1966some} 
are determined only up to an orthogonal transformation. 
Similar group symmetries arise in molecular graph 
generation~\cite{hoogeboom2022equivariant, xu2022geodiff, igashov2024equivariant, abramson2024accurate, xuquotient, zhou2026rethinking}, 
where a rotated molecule or a re-indexed graph represents the same physical entity. 
Existing approaches to group symmetry fall into two categories: 
(i) augmenting the target distribution in the original space by assigning equal probability to all elements in an invariance class 
through data augmentation~\cite{abramson2024accurate} or modeling the score function on the symmetry manifold~\cite{xu2022geodiff}; 
(ii) removing the redundancy by collapsing each equivalence class to a single element, 
either through a quotient space~\cite{xuquotient} or canonicalization to a unique representative~\cite{zhou2026rethinking}.

DiffATS follows the philosophy from the second category, as it resolves the usual difficulty of describing equivalence classes by Grassmannian primitives. 
Xu et al.~\cite{xuquotient} formulate a quotient space diffusion model that projects the diffusion process 
onto the quotient space and lifts it back to the original space.
This requires computing geodesics and the exponential map in the Grassmann manifold~\cite{edelman1998geometry},
which is non-trivial and involves additional computation in the DM training and inference steps~\citep{de2022riemannian}. We instead keep the diffusion in the ambient space and select a 
canonical representative through OP alignment of factor matrices to a medoid anchor.
A concurrent work~\cite{zhou2026rethinking} similarly pursues canonicalization in $\mathrm{SO}(3)$ for molecule generation. In contrast, we study the geometric structure in the Grassmann manifold.

\noindent\textbf{Orthogonal Procrustes.} 
OP alignment has been widely used in the theoretical
analysis of nonconvex low-rank factorization, where it resolves the rotational
non-identifiability of matrix factors by comparing estimates up to an orthogonal
transformation~\cite{tu2016low, zhang2018fast, ma2024convergence}.
These works use Procrustes alignment as an analytical tool for proving optimization and recovery guarantees. 
In contrast, we use it to construct canonicalized factor representations for diffusion model training.

\section{Background} \label{sec:background}
This section reviews the preliminaries used throughout the paper.

\noindent \textbf{Stiefel and Grassmann Manifolds.}
Given integers $r < n$, the \emph{Stiefel manifold} is
$\mathrm{Stie}(n,r) = \{\Xb \in \RR^{n \times r} : \Xb^\top \Xb = I_r\}$,
the set of $n \times r$ matrices with orthonormal columns, where $I_r$ is the $r \times r$ identity matrix. The \emph{Grassmann manifold} $\mathrm{Gr}(n,r)$ is the set of $r$-dimensional subspaces of $\RR^n$, which can also be identified with the set of $n \times n$ rank-$r$ projection matrices.

\noindent \textbf{Tucker Decomposition.}
The Tucker decomposition expresses a tensor as a core tensor multiplied by a factor matrix along each mode. 
We focus on the variant with orthonormal factor matrices, which is the form computed by standard algorithms 
such as the higher-order singular value decomposition (HOSVD)~\cite{de2000multilinear} and higher-order orthogonal 
iteration (HOOI)~\cite{de2000best}.

Following the convention of~\cite{kolda2009tensor}, the mode-$k$ product of a tensor $\cC \in \RR^{r_1 \times \cdots \times r_k \times \cdots \times r_d}$ with a matrix $\Ub_k \in \RR^{n_k \times r_k}$, denoted $\cC \times_k \Ub_k$, is the tensor in $\RR^{r_1 \times \cdots \times r_{k-1} \times n_k \times r_{k+1} \times \cdots \times r_d}$ with entries
\begin{align*}
  (\cC \times_k \Ub_k)_{i_1, \dots, i_{k-1}, j_k, i_{k+1}, \dots, i_d}
  = \sum_{i_k=1}^{r_k} \cC_{i_1, \dots, i_k, \dots, i_d} (\Ub_k)_{j_k, i_k},
\end{align*}
where $i_a \in \{1, \dots, r_a\}$ indexes mode $a$ of $\cC$, and $j_k \in \{1, \dots, n_k\}$ indexes the new mode-$k$ axis of the product. For $d=2$, this reduces to matrix CUR decomposition~\citep{mahoney2006tensor}: with $\cC \in \RR^{r_1 \times r_2}$ and $\Ub_k \in \RR^{n_k \times r_k}$, we have $\cC \times_1 \Ub_1 = \Ub_1 \cC$ and $\cC \times_2 \Ub_2 = \cC \Ub_2^\top$.

Given a tensor $\cX \in \RR^{n_1 \times \cdots \times n_d}$, the rank-$(r_1, \dots, r_d)$ Tucker decomposition seeks a core tensor $\cC \in \RR^{r_1 \times \cdots \times r_d}$ and factor matrices $\Ub_k \in \mathrm{Stie}(n_k, r_k)$ that minimize the reconstruction error $
  \min_{\cC, \Ub_1, \dots, \Ub_d} \|\cX - \cC \times_1 \Ub_1 \times_2 \Ub_2 \cdots \times_d \Ub_d\|_F^2,
  \quad $ under the constraint that $\ \Ub_k^\top \Ub_k = I_{r_k}$, $ k = 1, \dots, d.$
Both HOSVD and HOOI compute approximate solutions to this problem. 
We refer to any such solution $(\cC, \Ub_1, \dots, \Ub_d)$ as a Tucker decomposition of $\cX$ with rank $(r_1, \dots, r_d)$.

\noindent \textbf{Score-matching Diffusion Models.}
A diffusion model~\cite{song2020score} learns the data distribution $p_0$ through a pair of forward
and reverse-time stochastic differential equations (SDEs). The forward SDE $d\xb = \fb(\xb, t)dt + g(t) d\wb$
gradually corrupts a sample $\xb(0) \sim p_0$ into noise as $t$ grows from $0$ to $T$,
where $\wb$ is a standard Wiener process.
Let $p_t$ denote the marginal distribution of $\xb(t)$. By Anderson~\cite{anderson1982reverse}, the reverse-time SDE $d\xb = [-\fb(\xb, t) - \frac{g(t)^2+\sigma(t)^2}{2} \nabla_{\xb} \log p_t(\xb)] dt + \sigma(t) d\bar\wb$
recovers $p_0$ when integrated backward from $\xb(T) \sim p_T$~\citep{karras2022elucidating}.
The only unknown term in the reverse-time SDE is the \emph{score function} $\nabla_{\xb} \log p_t(\xb)$, 
the gradient of the log marginal diffused density at $t$ with respect to $\xb$. 
Score matching trains a neural network $s_\theta(\xb, t)$ 
to approximate the score, and sampling is performed by simulating the reverse-time SDE 
with $s_\theta$ in place of $\nabla_{\xb} \log p_t$.

\section{Search for Tensor Primitives}

Given a tensor $\cX$ with low Tucker rank, 
our goal is to construct a parameter-efficient primitive $\Phi(\cX)$ 
amenable to generative modeling. 
An effective tensor primitive should satisfy three desiderata: 
(i) it is expressive enough to represent a rich class of low-rank tensors; 
(ii) it reduces the parameter count of the original tensor; 
and (iii) it is compatible with standard diffusion model training and sampling.

\subsection{Why Are Natural Low-rank Primitives Undesirable?}
An unpretentious choice is to take the Tucker factors themselves as the primitive:
$\Phi(\cX) = (\cC, \Ub_1, \dots, \Ub_d)$, where $\cC$ is the core tensor and
$\Ub_1, \dots, \Ub_d$ are the factor matrices. We refer to this representation
as the \emph{natural primitive}. Since HOSVD and HOOI provide efficient
algorithms for computing such factors, the natural primitive appears appealing
at first glance.

However, the Tucker decomposition is non-unique. For any orthogonal matrix
$\Qb \in \mathrm{O}(r_k) = \{\Qb \in \RR^{r_k \times r_k} : \Qb^\top \Qb = I_{r_k}\}$
and any mode $k$, replacing $(\cC, \Ub_k)$ with $(\cC \times_k \Qb^\top, \Ub_k \Qb)$
leaves the reconstructed tensor unchanged:
  $(\cC \times_k \Qb^\top) \times_1 \Ub_1 \cdots \times_k (\Ub_k \Qb) \cdots \times_d \Ub_d
  = \cC \times_1 \Ub_1 \cdots \times_d \Ub_d$.
Thus a single tensor corresponds to infinitely many equivalent natural
primitives. These primitives all reconstruct $\cX$ exactly, yet can lie
arbitrarily far apart in the primitive space. This poses a problem for
diffusion training: if a single $\cX$ is encoded by a dispersed cloud of
primitives, the diffusion model must spread mass across all of them,
artificially inflating the support of the target distribution.

We illustrate this dispersion with a scalar toy example. Let $X = UV$ with
$U, V \in \RR$. For any fixed $x$, the equivalent factorizations form the
hyperbola $\{(u, v) : uv = x\}$, which is unbounded, even when $X$ has
bounded support.

\begin{figure}[t]
    \centering
    \includegraphics[width=1.00\textwidth]{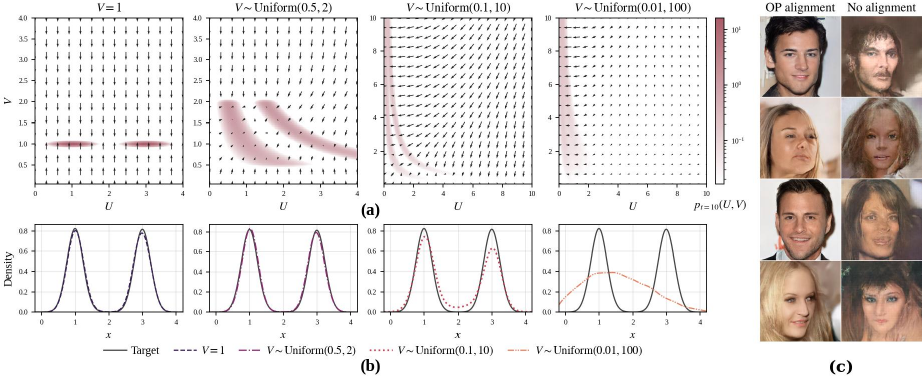}\hfill
    \caption{\textbf{Effect of OP alignment.}
    \textbf{(a)} Density and learned score fields $\nabla_{(U, V)} \log p_t(U, V)$ at diffusion step $t=10$
    for four sampling distributions of $V$ (arrows normalized within each panel).
    Pink shading shows the density $p_{10}(U, V)$, with darker color indicating higher density.
    \textbf{(b)} Sample KDE of $X$ generated by the trained $s_{\theta}$.
    \textbf{(c)} Samples generated by DMs trained on CelebA-HQ ($1024 \times 1024$),
    with and without OP alignment. Both models share the same training budget.}
    \label{fig:toy_model}
\end{figure}

To evaluate diffusion generative modeling on this primitive, we sample
$N = 20{,}000$ points $\{x_i\}_{i=1}^N$ from the two-component Gaussian mixture:
$p(X) = \tfrac{1}{2} \mathcal{N}(1, 0.2^2) + \tfrac{1}{2} \mathcal{N}(3, 0.2^2)$.
For each $x_i$, we draw $v_i \sim \mathrm{Uniform}(a, b)$ and set
$u_i = x_i / v_i$, producing equivalent factorizations $\{(u_i, v_i)\}_{i=1}^N$.
We then train a two-layer MLP score network $s_\theta(u, v, t)$ with the
denoising score-matching objective~\citep{ho2020denoising} to approximate
$\nabla \log p_t(u, v)$.

Fig.~\ref{fig:toy_model} (a) shows that increasing the range of $V$ spreads the
training distribution over a much larger region of the primitive space.
The score field becomes harder to learn, and the generated $X = UV$ deviates
from the target mixture distribution, as is evidenced in Fig.~\ref{fig:toy_model} (b). The failure mode arises from the
unnecessary variance in the natural $(U, V)$ primitive:
the variance of $U$ and $V$ can be orders of magnitude larger than that of $X$.



This observation suggests a guiding principle: \emph{a useful primitive should not 
introduce variation within an equivalence class}. In the scalar example, 
anchoring $V = 1$ and encoding $X$ as $(U, V) = (X, 1)$ produces a primitive 
distribution as concentrated as $p(X)$, yielding an accurate sampler (Fig.~\ref{fig:toy_model}(a,b)). Although the scalar model is deliberately simple, the same mechanism appears in Tucker factors. Fig.~\ref{fig:toy_model} (c) illustrates this on CelebA-HQ. These observations motivate the remainder of this section, which extends the anchoring idea from the scalar model to the multilinear setting via OP alignment.

\subsection{Orthogonal Procrustes Provides Stable Representations}

The toy example highlights the importance of controlling dispersion of primitives. We therefore seek a primitive map $\Phi$ that is
single-valued and continuous with respect to the underlying tensor.

\noindent\textbf{Orthogonal Procrustes Alignment.}
We fix the rotation ambiguity in Tucker factorization through OP.
Given orthonormal matrices $\Ab, \Bb \in \stie(n_1, n_2)$, OP finds the orthogonal matrix
$\Qb^\star \in \arg\min_{\Qb \in \mathrm{O}(n_2)} \|\Ab \Qb - \Bb\|_F^2$.
The solution to the OP problem admits a closed form~\cite{schonemann1966generalized}:
if $\Ab^\top \Bb = \Lb \bLambda \Rb^\top$ is an SVD, then $\Qb^\star = \Lb \Rb^\top$ is an optimal solution. OP alignment allows us to choose a canonical representative for each $r$-dimensional
subspace of $\RR^n$, relative to a fixed anchor $\Vb_0$. The following proposition
formalizes this selection.

\begin{proposition}
    \label{prop:uniqueness}
    Let $\Vb, \Vb_0 \in \stie(n, r)$ satisfy $\Vb_0^\top \Vb \Vb^\top \Vb_0 \succ 0$. Then the optimization problem
      $\arg\min_{\tilde\Vb : \tilde\Vb \tilde\Vb^\top = \Vb\Vb^\top} \|\Vb_0 - \tilde\Vb\|_F^2$
    has a unique solution $\Vb \Qb^\star$, where $\Qb^\star$ is the optimal solution to the OP problem $\arg\min_{\Qb \in \mathrm{O}(r)} \|\Vb_0 - \Vb \Qb\|_F^2$.
\end{proposition}

We denote the unique solution by $\op(\Vb, \Vb_0)$. 
Geometrically, Prop.~\ref{prop:uniqueness} selects the orthonormal 
frame of $\mathrm{span}(\Vb)$ that is closest to the anchor $\Vb_0$. 
Note that $\op(\Vb, \Vb_0)$ depends on $\Vb$ only through $\mathrm{span}(\Vb)$: $\op(\Vb \Qb, \Vb_0) = \op(\Vb, \Vb_0)$ for any $\Qb \in \mathrm{O}(r)$. 
Thus, $\op(\Vb, \Vb_0)$ defines a canonical representative of the equivalence class $[\Vb] = \{\Vb \Qb : \Qb \in \mathrm{O}(r)\}$, 
which is in one-to-one correspondence with a point on the Grassmann manifold. We define
$\mathrm{gr}(n, r \mid \Vb_0) := \{\op(\Vb, \Vb_0) : \Vb \in \stie(n, r),\ \Vb_0^\top \Vb \Vb^\top \Vb_0 \succ 0\}$.
Then $\mathrm{gr}(n, r \mid \Vb_0)$ has a one-to-one correspondence with the subset of $\mathrm{Gr}(n, r)$
consisting of subspaces having no direction perpendicular to $\mathrm{span}(\Vb_0)$.
Empirically, this assumption holds on real data: a statistical analysis of the principal angles between
factor subspaces and their medoid anchors shows that the maximum angle stays below $90^\circ$
on every tested mode of our PDE datasets (App.~\ref{app:anchor_overlap}).

\noindent\textbf{Reduction in Dispersion.}
We next quantify the contraction induced by OP alignment.
We start with a random subspace model. Let $\Vb_0, \Vb \in \mathrm{Stie}(p, r)$
be two independent, uniformly distributed orthonormal matrices with respect to the Haar measure. The expected squared distance between
them is
$\mathbb{E}\|\Vb - \Vb_0\|_F^2
  = \mathbb{E}\|\Vb\|_F^2 + \mathbb{E}\|\Vb_0\|_F^2 - 2\mathrm{Tr}(\mathbb{E}[\Vb^\top \Vb_0])
  = 2r$.
In contrast, the following proposition characterizes their asymptotic
distance after OP alignment.

\begin{proposition}[Asymptotic Procrustes distance between random subspaces] 
\label{prop:procrustes_reduction}
Let $\Vb_0,\Vb\in\mathrm{Stie}(p,r)$ be two independent Haar-distributed
orthonormal matrices. Suppose that $p,r\to\infty$ and $\frac{p}{r}\to c>1$. Let
$\Qb^\star \in \argmin_{\Qb\in\mathrm{O}(r)}
\|\Vb_0-\Vb\Qb\|_F^2 $
be an optimal OP solution and $\tilde\Vb=\Vb\Qb^\star$. Then $\|\Vb_0-\tilde\Vb\|_F^2
$ will converge to $
2r(1-\ell(c))+o(r)$, 
where
$\ell(c)
=
\frac{c}{\pi}
\left(
\sqrt b
-
\sqrt{1-b}
\arctan\sqrt{\frac{b}{1-b}}
\right)
+
\max\{2-c,0\}$, and $b=\frac{4(c-1)}{c^2}$.
\end{proposition}

Prop.~\ref{prop:procrustes_reduction} shows that OP alignment reduces the squared
distance by $\approx 2r\ell(c)$. For example, when $c = p/r = 4$, $\ell(c) \approx 0.44$,
so OP alignment removes about $44\%$ of the unaligned baseline distance.

\begin{wrapfigure}{r}{0.5\textwidth}
    \centering
    \vspace{-12pt}
    \includegraphics[width=0.48\textwidth]{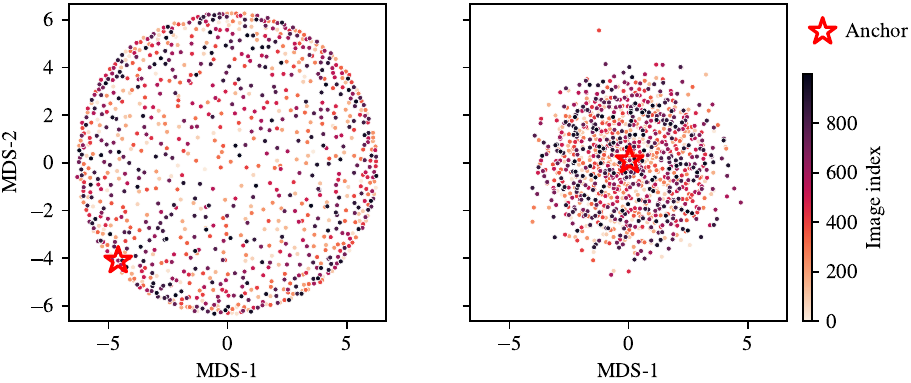}
    \caption{\textbf{OP alignment of factor matrices makes diffusion models easier to train.}
    \textbf{(Left):} Without alignment.
    \textbf{(Right):} With OP alignment.}
    \vspace{-10pt}
    \label{fig:procrustes_alignment}
\end{wrapfigure}
To visualize the dispersion reduction on realistic data, we randomly select 
$N = 1000$ images from CelebA-HQ, patchify each into $32 \times 32$ patches, 
and compute the SVD of the patch matrix on the first channel, yielding a 
right singular matrix $\Vb_i$ per image. We select the anchor as the medoid of these 
$N = 1000$ right singular matrices, following Step 2 of Sec.~\ref{subsec:mgp}. 
Fig.~\ref{fig:procrustes_alignment} visualizes the right singular matrices
under multidimensional scaling (MDS)~\citep{kruskal1978multidimensional}, with each marker corresponding to one 
$\Vb_i$ projected to two dimensions. Before alignment, the matrices are 
widely scattered. After OP alignment to a common anchor, they concentrate 
around the anchor, indicating that alignment substantially reduces the 
variation in the raw singular subspaces.



\subsection{Matrix Grassmannian Primitives (MGPs)}
\label{subsec:mgp}

We now leverage OP alignment to construct the mapping $\Phi$ from a matrix to its primitives. 
The result is named the \emph{matrix Grassmannian
primitive}. We take three steps to implement $\Phi$.

\noindent\textit{\underline{Step 1: Singular value decomposition.}} 
Consider a rank-$r$ matrix $\Mb \in \RR^{n_1 \times n_2}$ with singular value decomposition 
$\Mb = \Ub \bSigma \Vb^\top$, where $\Ub \in \mathrm{Stie}(n_1, r)$, 
$\Vb \in \mathrm{Stie}(n_2, r)$, and $\bSigma \in \RR^{r \times r}$ is diagonal. 
In real-world applications, observation matrices are rarely exactly low-rank; however, if
their singular values decay rapidly, they can be well-approximated by a low-rank matrix. 
In that case, we apply truncated SVD ($r$-SVD) to retain the top $r$ singular vectors.

\noindent\textit{\underline{Step 2: Anchor matrix selection.}} 
We choose the anchor as the training sample whose right singular subspace is most representative 
of the empirical distribution. Specifically, given candidate right singular subspaces 
$\{\Vb_i\}_{i=1}^N$, the medoid index is $i^\star \in \arg\max_{i \in [N]} \sum_{j=1}^N \|\Vb_i^\top \Vb_j\|_F^2$. The quantity $\|\Vb_i^\top \Vb_j\|_F^2$ measures the overlap between two subspaces, 
with larger values indicating closer alignment. Hence setting $\Vb_0 = \Vb_{i^\star}$
creates the largest average overlap with the training subspaces.
Empirically, this criterion is effective: replacing the medoid with the least representative
candidate from the same pool reduces the subspace overlap and worsens every generation metric (App.~\ref{app:anchor_overlap}).

\noindent\textit{\underline{Step 3: Primitive construction.}} 
For each data matrix $\Mb_i$, we apply OP to align its right singular matrix $\Vb_i$ 
to the anchor $\Vb_0$, producing the optimal rotation $\Qb_i^\star$ such that 
$\Vb_i \Qb_i^\star$ is closest to $\Vb_0$. We obtain the pair 
$\Phi(\Mb_i) = (\Ub_i \bSigma_i \Qb_i^\star, \Vb_i \Qb_i^\star)$. We call this primitive 
a \emph{matrix Grassmannian primitive} (MGP) as it exploits the inherent low-rank structure in matrices.

The total parameter count of $\Phi(\Mb_i)$ is $(n_1 + n_2) r$, 
compared to $n_1 n_2$ for the original matrix $\Mb_i$. 
The MGP thus achieves substantial parameter reduction when $r \ll \min(n_1, n_2)$, 
while preserving exact reconstruction for rank-$r$ matrices.

\noindent\textbf{Analysis of the Matrix Grassmannian Primitive.}
The MGP construction explicitly respects the low-rank structure 
in matrices while preventing the dispersion of natural primitives. 
We justify the MGP $\Phi$ by characterizing its continuity.

\begin{theorem}\label{thm:matrix-case}
    For $r < n_1, n_2$, let $\cM(n_1, n_2, r \mid \Vb_0) = \{\Mb \in \RR^{n_1 \times n_2} : \mathrm{rank}(\Mb) = r, \Vb_0^\top \Mb^\top \Mb \Vb_0 \succ 0\}$ 
    denote rank-$r$ matrices whose principal angles with $\mathrm{span}(\Vb_0)$ 
    are all strictly smaller than $\pi/2$, and $\RR_*^{n_1 \times r} = \{\tilde\Ub \in \RR^{n_1 \times r} : \mathrm{rank}(\tilde\Ub) = r\}$. Then the map
    \begin{align}
        \Phi(\cdot \mid \Vb_0): \mathcal M(n_1, n_2, r \mid \Vb_0) &\to \RR_*^{n_1 \times r} \otimes \mathrm{gr}(n_2, r \mid \Vb_0)\nonumber, \quad \Mb \mapsto (\Mb \tilde\Vb, \tilde\Vb)
    \end{align}
    is a homeomorphism, where $\tilde \Vb=\op(\Vb,\Vb_0)$ is solved via the OP 
    alignment introduced in Prop.~\ref{prop:uniqueness}, and $\Vb$ comes from the SVD $\Mb=\Ub\bSigma\Vb^\top$.
\end{theorem}

Thm.~\ref{thm:matrix-case} shows that the MGP is a continuous, 
one-to-one parametrization of rank-$r$ matrices. 
Hence, the primitive map resolves the rotational degeneracy in matrix factorization and introduces 
no additional variation in the primitive. The full proof is deferred to App.~\ref{proof:matrix-case}.

We remark on the requirement $\Vb_0^\top \Mb^\top \Mb \Vb_0 \succ 0$. 
For image and video data, the right singular subspaces empirically 
concentrate around a small number of dominant directions (see Fig.~\ref{fig:procrustes_alignment}). 
Selecting $\Vb_0$ by the medoid criterion makes $\Vb_0^\top \Vb_i \Vb_i^\top \Vb_0$ 
more strongly positive definite than an arbitrary choice, either by attaining a larger minimum eigenvalue or a larger trace.

\subsection{Tensor Grassmannian Primitives (TGPs)} \label{subsec:tgps}
We extend the Grassmannian primitive construction to tensors. Consider $d$-th order tensors $\{\cX_i\}_{i=1}^N$ with each $\cX_i \in \RR^{n_1 \times \cdots \times n_d}$. The TGP is constructed in three steps parallel to the matrix case.

\noindent\textit{\underline{Step 1: Tucker decomposition.}} Apply HOOI to each $\cX_i$ to obtain
$\cX_i = \cC_i \times_1 \Ub_{1,i} \times_2 \Ub_{2,i} \cdots \times_d \Ub_{d,i}$,
where $\cC_i \in \RR^{r_1 \times \cdots \times r_d}$ is the core tensor and $\Ub_{k,i} \in \stie(n_k, r_k)$ 
is the factor matrix along mode $k$ for $k = 1, \dots, d$. Since Tucker decomposition is non-unique,
we fix only one such decomposition.

\noindent\textit{\underline{Step 2: Anchor matrix selection.}} Select anchor matrices for modes in $\cI_l = \{l, l+1, \dots, d\}$ for some $l \geq 2$. For each $k \in \cI_l$, find the medoid $\Ub_{k,0}$ from $\{\Ub_{k,i}\}_{i=1}^N$ via the medoid criterion of Sec.~\ref{subsec:mgp}.

\noindent\textit{\underline{Step 3: Primitive construction.}} For each $k \in \cI_l$ and $i$, apply OP to align $\Ub_{k,i}$ to the anchor $\Ub_{k,0}$, yielding $\tilde\Ub_{k,i} = \op(\Ub_{k,i}, \Ub_{k,0})$. The TGP is
$\Phi(\cX_i) = (\cX_i \times_l \tilde\Ub_{l,i}^\top \cdots \times_d \tilde\Ub_{d,i}^\top,\ \tilde\Ub_{l,i}, \dots, \tilde\Ub_{d,i})$.

DiffATS then trains DMs on $\{\Phi(\cX_i)\}_{i=1}^N$. 
At inference, the DM generates $(\hat\cG, \hat\Ub_l, \dots, \hat\Ub_d)$ by DDIM~\citep{ddim}, 
and the generated tensor is reconstructed as $\hat\cX = \hat\cG \times_l \hat\Ub_l \cdots \times_d \hat\Ub_d$ (see Fig.~\ref{fig:pipeline}).

\noindent\textbf{Analysis of the Tensor Grassmannian Primitive.}
The TGP construction is an intuitive extension of the MGP. However, since Tucker decomposition is multilinear (rather than bilinear, as in matrix factorization), the theoretical analysis of the TGP is more delicate. The major challenge is that tensors do not admit a single ``rank'' notion with the same stability properties as matrix rank~\citep{kolda2009tensor,schaefer2018complexity}.

To address this, we adopt the technique of \cite{zhang2018tensor}, 
which replaces tensor rank by the ranks of its mode-$k$ unfoldings. 
More specifically, for a tensor $\cX \in \RR^{n_1 \times n_2 \times \cdots \times n_d}$, we define an operator $\sk$ such that
$\sk(\cX)$ is an $n_k \times n_k$ matrix whose $(\alpha, \beta)$-th entry is $
    [\sk(\cX)]_{\alpha\beta} = \sum_{i_1 = 1}^{n_1} \cdots \sum_{i_{k-1} = 1}^{n_{k-1}} \sum_{i_{k+1} = 1}^{n_{k+1}} \cdots \sum_{i_d = 1}^{n_d} \cX_{i_1, \dots, i_{k-1}, \alpha, i_{k+1}, \dots, i_d} \cdot \cX_{i_1, \dots, i_{k-1}, \beta, i_{k+1}, \dots, i_d}$. 
We define the mode-$k$ rank of $\cX$ as $\mathrm{rank}(\sk(\cX))$.

\begin{theorem}[TGP Homeomorphism, informal]
\label{thm:tgphomeomorphism}
Fix $l \in \{1, \dots, d\}$ and mode-wise anchor matrices
$\Ub_{l,0}, \dots, \Ub_{d,0}$. 
The TGP map $\Phi(\cdot \mid \Ub_{l,0}, \dots, \Ub_{d,0})$ is a homeomorphism
between the set of tensors $\cX \in \RR^{n_1 \times \cdots \times n_d}$
such that, for every $k = l, \dots, d$,
$\mathrm{rank}(\sk(\cX)) = r_k$ and the anchor-overlap condition
$\Ub_{k,0}^\top \sk(\cX) \Ub_{k,0} \succ 0$ holds, and 
a product space of a coefficient tensor space and $(d - l + 1)$ Grassmann manifolds.
The formal statement and proof are deferred to 
App.~\ref{proof:tgphomeomorphism}.
\end{theorem}

Thm.~\ref{thm:tgphomeomorphism} justifies the construction of TGPs as stable parametrizations 
for anchored low-multilinear-rank tensors. 
As long as the factor matrices of $\cX$ have nontrivial overlap with the anchor matrix, 
the map $\Phi$ is a continuous bijective function. 
The bijectivity of $\Phi$ ensures that the TGP retains the representation power 
of Tucker decomposition while removing the rotational ambiguity. 
The continuity of $\Phi$ further ensures that the primitive distribution 
$p(\Phi(\cX))$ inherits the local geometry of the data distribution $p(\cX)$, 
so a diffusion model trained to match the score $\nabla \log p_t(\Phi(\cX))$ 
in the primitive space effectively characterizes $p(\cX)$ through inverse mapping.


\section{Experiments}
In this section, we evaluate DiffATS on CelebA-HQ~\cite{karras2017progressive}, 
Moving MNIST~\cite{srivastava2015unsupervised}, 
and four PDE datasets: 1-d Burgers, 1-d reaction-diffusion, 2-d Burgers, 
and 2-d K\'arm\'an vortex street. See App.~\ref{app:additional_results} 
for generated samples across all baselines and datasets.

We apply MGPs to matrix-structured data (CelebA-HQ images and 1-d PDEs) 
and TGPs to tensor-structured data (Moving MNIST videos and 2-d PDEs). 
For each dataset, we choose the rank to keep the average relative $\ell_2$ reconstruction error 
smaller than a threshold. 
For PDE data, where physical fidelity matters, 
we select the threshold from $\sim10^{-6}$ to $\sim 10^{-3}$ (see Tab.~\ref{tab:reconstruction_error}). 
The resulting compression factors range from $3.9\times$ to $210\times$ across datasets 
(see Tab.~\ref{tab:compression_ratio}), and the diffusion model is trained directly 
on the compressed MGPs and TGPs. 
The image and video experiments are unconditional. The PDE experiments 
are conditional: given the initial frame, the model generates the subsequent trajectory. Our code is available at
\url{https://github.com/JinhuaLyu/DiffATS}.

\begin{wraptable}{r}{0.35\textwidth}
\centering
\vspace{-15pt}
\setlength{\tabcolsep}{-2pt}
\small
\caption{\textbf{Per-sample compression ratio across the six benchmark datasets used by DiffATS.}}
\label{tab:compression_ratio}
\begin{tabular}{l r}
\toprule
Dataset & Compression ratio \\
\midrule
CelebA-HQ              & $16.00$ \\
Moving MNIST           & $3.94$ \\
1-d Burgers            & $6.67$ \\
1-d reaction-diffusion & $13.33$ \\
2-d Burgers            & $210.31$ \\
2-d K\'arm\'an         & $137.41$ \\
\bottomrule
\end{tabular}
\vspace{-15pt}
\end{wraptable}

We compare DiffATS with four baselines:
(i) DCTdiff~\cite{ning2024dctdiff}, which trains DMs in the DCT frequency space;
(ii) SDIFT~\cite{chen2025generating}, which trains DMs on Tucker cores with shared factor bases;
(iii) FNO~\cite{li2020fourier}, a neural operator for learning solution maps of parametric PDEs; and
(iv) AvgPool, which trains DMs on spatially downsampled data.
Note that we focus on generative baselines that compress information using lower-dimensional representations \emph{without autoencoders}.
We further conduct ablation studies on CelebA-HQ to evaluate our low-rank primitives and the OP alignment.

\noindent\textbf{Settings.}
For each dataset, we use the same number of training epochs and the same batch size across all methods, 
with $250$ sampling steps. 
All experiments use a DiT backbone~\cite{peebles2023scalable} 
(our method is compatible with other DM architectures). 
To compare representation choices fairly, we match the per-sample compression ratio across DCTdiff, SDIFT, AvgPool, and DiffATS, so that all four methods operate on representations of comparable dimensionality. FNO is the only exception: as a neural operator, it operates directly in the original data space and does not admit an analogous compression step. Full experimental details are provided in App.~\ref{app:exp_settings}.

\noindent\textbf{Preprocessing Cost Comparison.}
We compare the computational cost of DiffATS preprocessing with that of
autoencoder pretraining on 2-d Burgers.
DiffATS preprocessing takes approximately 40 min on 16 CPU cores,
without GPU use or neural network training, whereas autoencoder pretraining
requires 4.5 h on a single H100 GPU.
Detailed runtime and memory measurements and baseline settings
are provided in App.~\ref{app:ae_comparison}.

\subsection{Image and Video Synthesis}
Image and video experiments provide visually interpretable tests of representation quality and allow comparison with established generative metrics.

\noindent \textbf{Datasets.}
The CelebA-HQ dataset~\cite{karras2017progressive} 
consists of $30{,}000$ images of celebrity faces at a resolution of $1024 \times 1024$. 
We apply $32 \times 32$ channel-wise patchification (see App.~\ref{app:patchification}). 
Patchification consistently reduces the SVD truncation error at the same rank on CelebA-HQ, as validated by the ablation in Tab.~\ref{tab:svd_patchify_ablation}.
The Moving MNIST dataset~\cite{srivastava2015unsupervised} 
consists of $20{,}000$ video clips of moving handwritten digits, 
each with $20$ frames at $64 \times 64$ resolution. We follow the DiffATS pipeline 
in Fig.~\ref{fig:pipeline} directly.

\noindent \textbf{Metrics.}
For unconditional image synthesis, we report the Fr\'echet Inception Distance (FID)~\cite{heusel2017gans} 
with both Inception-V3~\cite{szegedy2016rethinking} and DINOv2~\cite{oquab2023dinov2} 
embeddings; the latter follows~\cite{stein2023exposing} for a less biased evaluation. 
For video generation, we report the standard Fr\'echet Video Distance (FVD)~\cite{unterthiner2018towards}.

\noindent\textbf{Results.}
Generated samples are shown in Figs.~\ref{fig:celeba_samples} and~\ref{fig:moving_mnist_samples} (App.~\ref{app:additional_results}). Tab.~\ref{tab:celeba} shows that DiffATS attains the best FID and FVD among autoencoder-free methods, validating the representational benefits of MGPs and TGPs. DiffATS does not aim to match state-of-the-art image and video synthesis models that rely on large pretrained autoencoders. Instead, the focus is on spatiotemporal physical fields, where such autoencoders are typically unavailable and low-rank structure is more pronounced.

\noindent \textbf{Ablation study.} We further ablate the key design choices of DiffATS: 
data-dependent Tucker bases, OP alignment, and factor-space augmentation. 
The full DiffATS uses data-dependent bases with OP alignment and without augmentation. 
``Shared bases'' uses global Tucker bases estimated from the training set and 
trains only the projected coefficients. ``Data-dep. bases w/ aug.'' 
applies factor-space augmentation to the data-dependent bases. ``Data-dep. bases w/o align.'' 
trains on unaligned data-dependent Tucker factors without augmentation. 
As shown in Tab.~\ref{tab:celeba}, data-dependent bases with OP alignment notably 
improve generation quality, while factor-space augmentation severely degrades it; 
the full DiffATS attains the lowest FID under both Inception-V3 and DINOv2 evaluations.

\begin{table}[t]
\centering
\setlength{\tabcolsep}{3pt}
\small
\caption{\textbf{Generation quality on image and video benchmarks.} Both datasets use
10{,}000 generated samples and 10{,}000 real samples for evaluation.
The bottom block compares alternative factor-space designs of DiffATS on CelebA-HQ for ablation.}
\label{tab:celeba}
\label{tab:mmnist_main}
\resizebox{\textwidth}{!}{%
\begin{tabular}{l cc c}
\toprule
& \multicolumn{2}{c}{CelebA-HQ ($1024\times 1024$, image)} & Moving MNIST ($64\times 64$, video) \\
\cmidrule(lr){2-3} \cmidrule(lr){4-4}
Method & FID (Inception-V3) $\downarrow$ & FID (DINOv2) $\downarrow$ & FVD $\downarrow$ \\
\midrule
DCTdiff~\cite{ning2024dctdiff}              & 52.22             & 787.11            & 555.36 \\
SDIFT~\cite{chen2025generating}             & 309.50            & 2455.57                & 1183.69 \\
AvgPool                                     & 32.89             & 559.67            & 487.82 \\
DiffATS (ours)                              & $\mathbf{23.57}$  & $\mathbf{381.58}$ & $\mathbf{415.42}$ \\
\midrule
\multicolumn{4}{l}{\emph{Ablations of DiffATS (CelebA-HQ only)}} \\
Shared bases                  & $24.23$  & $455.11$  & \\
Data-dep. bases w/ aug.       & $339.01$ & $2550.36$ & \\
Data-dep. bases w/o align.    & $76.94$  & $1175.78$ & \\
\bottomrule
\end{tabular}%
}
\end{table}

\subsection{PDE Trajectory Prediction}
The merits of DiffATS shine on parametric PDE trajectory synthesis. Many PDE solutions admit accurate low-rank approximations in smooth regimes~\citep{guo2025interpolating,guo2025tensor}, 
making them a natural fit for low-rank primitive-based diffusion modeling. In this section, we consider the task of solution synthesis, where we train DMs to generate PDE solutions from initial conditions.

\noindent \textbf{Datasets.}
We use four PDE datasets. 
The matrix datasets (1-d Burgers and 1-d reaction-diffusion) 
have size $1024 \times 201$, and the tensor datasets (2-d Burgers and 2-d K\'arm\'an vortex street) 
have size $128 \times 128 \times 201$. 
Each trajectory consists of $201$ frames: 
the first is the initial condition, 
and the remaining $200$ are the target trajectory. 
Each dataset contains $10{,}000$ training and $500$ test trajectories.

\noindent \textbf{PDE with MGPs.}
For matrix data, 
we split each trajectory into the initial frame $\Mb_0 \in \RR^{n_1}$ (the condition) 
and the subsequent frames $\Mb_{1:T} \in \RR^{n_1 \times T}$ (the target). 
We construct an MGP $\Phi(\Mb_{1:T})$ following Sec.~\ref{subsec:mgp}, 
anchoring the left singular vectors with $\Ub_0 \in \RR^{n_1 \times r}$. 
The initial frame $\Mb_0$ is encoded as $\Phi(\Mb_0)$ using the same anchor. 
A conditional diffusion model~\cite{ho2020denoising, ho2022cascaded} 
is trained on $p_\theta\bigl(\Phi(\Mb_{1:T}) \mid \Phi(\Mb_0)\bigr)$, 
with $\Phi(\Mb_0)$ injected into the denoising network as conditional inputs.

\noindent \textbf{PDE with TGPs.}
Similar to the 1-d case, we split each trajectory $\cX \in \RR^{n_1 \times n_2 \times (T+1)}$ 
into the condition $\Xb_0 \in \RR^{n_1 \times n_2}$ and the target $\cX_{1:T} \in \RR^{n_1 \times n_2 \times T}$. 
We construct the TGP $\Phi(\cX_{1:T}) = (\tilde{\mathcal{G}}, \tilde{\Ub}_2, \tilde{\Ub}_3)$ 
following Sec.~\ref{subsec:tgps}, anchoring the mode-2 and mode-3 factor matrices 
with $(\Ub_{2,0}, \Ub_{3,0})$. The initial frame $\Xb_0$ is encoded as 
$\Phi(\Xb_0) = (\tilde{\Ub}_0, \tilde{\Vb}_0)$ via MGP, anchoring its right singular vectors with 
the same $\Ub_{2,0}$ so that condition and target share a consistent factor representation.

\noindent \textbf{Results.}
DiffATS outperforms all baselines in predictive metrics by a large margin on
all four PDE benchmarks (Tab.~\ref{tab:pde_main}),
reducing the relative $\ell_2$ error by roughly $9\times$ 
over the best baseline on 1-d reaction-diffusion. 
The improvement compared with SDIFT and DCTdiff suggests that
data-dependent Tucker bases used in DiffATS adapt to each trajectory's spatial patterns more flexibly.
Beyond these error metrics, we evaluate four physical diagnostics on 2-d Burgers,
namely energy dissipation, one-step PDE residual, periodic boundary consistency,
and long-horizon stability (App.~\ref{app:pde_physics}).
DiffATS performs best on all four, and its kinetic-energy trajectory error is
$6$ to $25\times$ smaller than those of the baselines.
Ablating OP alignment on 2-d Burgers and 1-d reaction-diffusion increases
the relative $\ell_2$ error by $4.4\times$ and $3.5\times$, respectively,
and worsens every physical diagnostic (App.~\ref{app:op_ablation_pde}),
consistent with the CelebA-HQ ablation in Tab.~\ref{tab:celeba}.
DiffATS also remains effective across ranks (App.~\ref{app:rank_sensitivity}): our method is robust to the rank selection. 
Although all baselines are autoencoder-free, we also compare DiffATS with
a latent diffusion model built on a pretrained 3-d convolutional
autoencoder as a reference (App.~\ref{app:ae_latent_diffusion}).
On 2-d Burgers, DiffATS attains a slightly lower relative $\ell_2$ error
together with lower boundary-consistency and final-frame errors,
while avoiding the 4.5 h of autoencoder pretraining (App.~\ref{app:ae_comparison}).

\begin{table}[t]
\centering
\setlength{\tabcolsep}{6pt}
\small
\caption{\textbf{Generation quality on PDE datasets.} Relative $\ell_2$ error between generated and ground-truth trajectories, averaged over $5$ random seeds. Full results with $\mathrm{Rel.~Err.}_1$ and $\mathrm{Avg.~rMSE}$ are in App.~\ref{app:pde_full_metrics} (Tab.~\ref{tab:pde_1d_full} and Tab.~\ref{tab:pde_2d_full}).}
\label{tab:pde_main}
\begin{tabular}{l cc cc}
\toprule
& \multicolumn{2}{c}{1-d} & \multicolumn{2}{c}{2-d} \\
\cmidrule(lr){2-3} \cmidrule(lr){4-5}
Method & Burgers & Reaction-Diffusion & Burgers & K\'arm\'an \\
\midrule
DCTdiff~\cite{ning2024dctdiff}    & $0.2953 \pm 0.0100$          & $0.0713 \pm 0.0001$          & $0.3241 \pm 0.0048$          & $0.5550 \pm 0.0710$ \\
SDIFT~\cite{chen2025generating}   & $0.3884 \pm 0.2600$          & $0.3406 \pm 0.0069$          & $0.7323 \pm 0.2030$          & $0.8654 \pm 0.1200$ \\
FNO~\cite{li2020fourier}          & $0.1913 \pm 0.1800$          & $0.3114 \pm 0.0069$          & $0.3423 \pm 0.3000$          & $0.3442 \pm 0.0190$ \\
AvgPool                           & $0.6602 \pm 0.0120$          & $0.3975 \pm 0.0038$          & $0.3747 \pm 0.0029$          & $0.6161 \pm 0.0020$ \\
DiffATS (ours)                    & $\mathbf{0.1203 \pm 0.0001}$ & $\mathbf{0.0085 \pm 0.0000}$ & $\mathbf{0.1830 \pm 0.0005}$ & $\mathbf{0.2233 \pm 0.0005}$ \\
\bottomrule
\end{tabular}
\end{table}

\section{Conclusion and Discussion}\label{sec:conclusion}

We presented DiffATS, a diffusion framework on OP-aligned low-rank primitives 
(MGPs and TGPs) that resolves Tucker's rotational ambiguity and outperforms 
chosen autoencoder-free baselines on unconditional and conditional generative tasks. Future research directions include multi-anchor OP alignment and dynamic rank selection to relax the assumptions in Thms.~\ref{thm:matrix-case} and~\ref{thm:tgphomeomorphism}.



\newpage
\begin{ack}
We are grateful to the AC and anonymous reviewers for their constructive feedback. This research is supported
by the U.S. National Science Foundation (NSF) under grant DMS-2610730 and partially by an Amazon
Research Award in Fall 2025. This research is also supported in part through NSF ACCESS and through the computational resources and staff contributions provided for the Quest high-performance computing facility at Northwestern University which is jointly supported by the Office of the Provost, the Office for Research, and Northwestern University Information Technology.
\end{ack}

\bibliographystyle{apalike}
\bibliography{ref}

\newpage
\appendix
\section{Proofs}
\subsection{Proof of Proposition~\ref{prop:uniqueness}}

\begin{proof}
    Since $\Vb \Vb^\top$ is the $n \times n$ orthogonal projection onto $\mathrm{span}(\Vb)$, the constraint $\tilde\Vb \tilde\Vb^\top = \Vb \Vb^\top$ forces $\tilde\Vb$ to have orthonormal columns spanning $\mathrm{span}(\Vb)$. Hence
    \[
        \{\tilde\Vb \in \RR^{n \times r} : \tilde\Vb \tilde\Vb^\top = \Vb \Vb^\top\}
        = \{\Vb \Qb : \Qb \in \mathrm{O}(r)\},
    \]
    and the constrained problem reduces to the OP problem
    \[
        \min_{\Qb \in \mathrm{O}(r)} \|\Vb_0 - \Vb \Qb\|_F^2.
    \]
    Since $\tilde \Vb=\Vb \Qb$ gives a one-to-one correspondence between $\tilde \Vb$ and $\Qb$, we only need the uniqueness of the optimal solution $\Qb^\star$ to the OP problem.

    Expanding $\|\Vb_0 - \Vb \Qb\|_F^2$ and using $\|\Vb \Qb\|_F^2 = \|\Vb_0\|_F^2 = r$ gives
    \[
        \|\Vb_0 - \Vb \Qb\|_F^2 = 2r - 2\mathrm{Tr}(\Vb_0^\top \Vb \Qb).
    \]
    Thus, the OP problem is equivalent to maximizing $\mathrm{Tr}(\Vb_0^\top \Vb \Qb)$ over $\Qb \in \mathrm{O}(r)$. Let $\Vb^\top \Vb_0 = \Lb \bLambda \Rb^\top$ be a compact SVD with $\bLambda = \mathrm{diag}(\lambda_1, \dots, \lambda_r)$, $\lambda_1 \geq \cdots \geq \lambda_r > 0$, and $\Lb,\Rb\in\mathrm O(r)$. Setting $\Mb = \Lb^\top \Qb \Rb \in \mathrm{O}(r)$,
    \[
        \mathrm{Tr}(\Vb_0^\top \Vb \Qb) = \mathrm{Tr}(\Rb \bLambda \Lb^\top \Qb) = \mathrm{Tr}(\bLambda \Mb) = \sum_{i=1}^r \lambda_i M_{ii}.
    \]
    Since $|M_{ii}| \leq 1$ for all $\Mb \in \mathrm{O}(r)$, the trace has a maximum $\sum_i \lambda_i$, achieved by $\Mb = I_r$, i.e., $\Qb^\star = \Lb \Rb^\top$~\cite{schonemann1966generalized}. This gives $\Vb \Qb^\star$ as a minimizer of the original problem.

    For the uniqueness, we need to prove that when the SVD $\Vb^\top\Vb_0=\Lb\bLambda \Rb^\top$ cannot uniquely determine the tuple $(\Lb,\bLambda,\Rb)$, the product $\Lb\Rb^\top$ is still uniquely determined. In fact, 
    \begin{align}\label{eq:Qoptimal}
        \Lb\Rb^\top=\Lb\bLambda \Rb^\top(\Rb\bLambda^{-1}\Rb^\top)=\Vb^\top\Vb_0(\Vb_0^\top\Vb\Vb^\top\Vb_0)^{-\frac12},
    \end{align}
    where we use $S^{-\frac12}$ for a positive definite matrix $S$ to denote the inverse of its square root, the unique positive definite matrix $T$ such that $T^2=S$.
    
\end{proof}

\subsection{Proof of Proposition~\ref{prop:procrustes_reduction}}
\begin{proof}
For any \(\Qb\in\mathrm{O}(r)\), since
\(\Vb_0^\top \Vb_0=\Vb^\top \Vb=I_r\), we have
\[
\|\Vb_0-\Vb\Qb\|_F^2
=
\mathrm{Tr}(\Vb_0^\top \Vb_0)
+
\mathrm{Tr}(\Qb^\top \Vb^\top \Vb \Qb)
-
2\mathrm{Tr}(\Vb_0^\top \Vb\Qb).
\]
Hence
\[
\|\Vb_0-\Vb\Qb\|_F^2
=
2r
-
2\mathrm{Tr}(\Vb_0^\top \Vb\Qb).
\]
Maximizing the last trace over \(\Qb\in\mathrm{O}(r)\) gives the
OP solution. If
$\Vb_0^\top \Vb = \Ub \bSigma \Wb^\top$
is a singular value decomposition, then one maximizer is
\[
\Qb^\star = \Wb \Ub^\top,
\]
and
\[
\max_{\Qb\in\mathrm{O}(r)}
\mathrm{Tr}(\Vb_0^\top \Vb\Qb)
=
\|\Vb_0^\top \Vb\|_*.
\]
Therefore,
\[
\|\Vb_0-\Vb\Qb^\star\|_F^2
=
2r-2\|\Vb_0^\top \Vb\|_*.
\]

It remains to compute the asymptotic value of
\(\|\Vb_0^\top \Vb\|_*/r\). Let
\[
\Pb_0=\Vb_0\Vb_0^\top,
\quad
\Pb=\Vb\Vb^\top
\]
be the two random rank-\(r\) orthogonal projectors in \(\mathbb{R}^p\).
The nonzero eigenvalues of
\[
\Pb_0\Pb\Pb_0
=
\Vb_0\Vb_0^\top \Vb\Vb^\top \Vb_0\Vb_0^\top
\]
coincide with the eigenvalues of
\[
\Vb_0^\top \Vb\Vb^\top \Vb_0
=
(\Vb_0^\top \Vb)(\Vb_0^\top \Vb)^\top.
\]
Consequently, if \(\mu_p\) denotes the empirical spectral distribution of
\(\Pb_0\Pb\Pb_0\), normalized by \(p\), then
\[
\frac{1}{p}\|\Vb_0^\top \Vb\|_*
=
\int_0^\infty \sqrt{x}d\mu_p(x).
\]

By Wachter's MANOVA law~\cite{kunisky2023generic} with parameters
\[
\alpha=\beta=\frac{1}{c},
\]
the empirical spectral distribution \(\mu_p\) converges weakly to
\[
d\mu(x)
=
\frac{\sqrt{x(b-x)}}{2\pi x(1-x)}
\mathbf{1}_{[0,b]}(x)dx
+
\max\left\{\frac{2}{c}-1,0\right\}\delta_1(x)
+
\left(1-\frac{1}{c}\right)\delta_0(x),
\]
where
\[
b
=
4\frac{1}{c}\left(1-\frac{1}{c}\right)
=
\frac{4(c-1)}{c^2}.
\]
Since \(x\mapsto \sqrt{x}\) is continuous and bounded on \([0,1]\), we obtain
\[
\frac{1}{p}\|\Vb_0^\top \Vb\|_*
\to
\int_0^\infty \sqrt{x}d\mu(x).
\]
The limiting integral equals
\[
\int_0^\infty \sqrt{x}d\mu(x)
=
\int_0^b
\frac{\sqrt{b-x}}{2\pi(1-x)}dx
+
\max\left\{\frac{2}{c}-1,0\right\}.
\]
A direct calculation gives
\[
\int_0^b
\frac{\sqrt{b-x}}{2\pi(1-x)}dx
=
\frac{1}{\pi}
\left(
\sqrt b
-
\sqrt{1-b}
\arctan\sqrt{\frac{b}{1-b}}
\right).
\]
Thus
\[
\frac{1}{p}\|\Vb_0^\top \Vb\|_*
\to
\frac{1}{\pi}
\left(
\sqrt b
-
\sqrt{1-b}
\arctan\sqrt{\frac{b}{1-b}}
\right)
+
\max\left\{\frac{2}{c}-1,0\right\}.
\]
Since \(p/r\to c\), it follows that
\[
\frac{1}{r}\|\Vb_0^\top \Vb\|_*
\to
\frac{c}{\pi}
\left(
\sqrt b
-
\sqrt{1-b}
\arctan\sqrt{\frac{b}{1-b}}
\right)
+
c\max\left\{\frac{2}{c}-1,0\right\}.
\]
Equivalently,
\[
\frac{1}{r}\|\Vb_0^\top \Vb\|_*
\to
\ell(c),
\]
where
\[
\ell(c)
=
\frac{c}{\pi}
\left(
\sqrt b
-
\sqrt{1-b}
\arctan\sqrt{\frac{b}{1-b}}
\right)
+
\max\{2-c,0\}.
\]
Combining this with
\[
\|\Vb_0-\Vb\Qb^\star\|_F^2
=
2r-2\|\Vb_0^\top \Vb\|_*
\]
yields
\[
\frac{1}{r}
\|\Vb_0-\Vb\Qb^\star\|_F^2
\to
2(1-\ell(c)).
\]
Equivalently,
\[
\|\Vb_0-\Vb\Qb^\star\|_F^2
=
2r(1-\ell(c))+o(r).
\]

Finally, we verify that \(\ell(c)\in(0,1)\). Since the singular values of
\(\Vb_0^\top\Vb\) lie in \([0,1]\), for every \(p,r\),
\[
0
\le
\frac{1}{r}\|\Vb_0^\top\Vb\|_*
\le
1.
\]
Passing to the limit gives \(0\le \ell(c)\le 1\). Moreover, for any
fixed \(c>1\), the Wachter limiting measure has nonzero continuous mass
on \((0,b)\), where \(b=4(c-1)/c^2>0\). Hence
\[
\ell(c)
=
c\int_0^\infty \sqrt{x}d\mu(x)
>
0.
\]
The inequality \(\ell(c)<1\) follows because the limiting measure is not
concentrated at \(x=1\) when \(c>1\). Thus
\[
0<\ell(c)<1.
\]
This proves the claim.

\end{proof}
\subsection{Proof of Theorem~\ref{thm:matrix-case}}\label{proof:matrix-case}
\begin{proof}
    As we have the uniqueness from Prop.~\ref{prop:uniqueness} of the $\tilde \Vb$ part, $\Phi(\cdot\mid\Vb_0)$ is injective.
    
    For surjectivity, we claim that each pair $(\Mb,\Vb)\in \mathbb R_*^{n_1\times r}\otimes\mathrm{gr}(n_2,r\mid\Vb_0)$ is the image of the matrix $\Mb \Vb^\top$. Given that $\Mb$ has full rank, the projection matrix of the row subspace of $\Mb \Vb^\top$ is just $\Vb\Vb^\top$. Then $\Vb\in\mathrm{gr}(n_2,r\mid\Vb_0)$ implies that the second part of $\Phi(\Mb\Vb^\top)$ is $\Vb$ itself, so $\Phi(\cdot\mid\Vb_0)$ is surjective and thus a bijection.

    The inverse map of $\Phi(\cdot\mid\Vb_0)$ is just a product map and is  continuous.
    So we only need to prove the continuity of the map $\Mb\to \tilde \Vb$, as this also implies the continuity of the map $\Mb\to \Mb \tilde \Vb$. 
    
    Since $\Mb$ has fixed rank, the row subspace of $\Mb$ continuously depends on $\Mb$~\cite{davis1970rotation}, as does the projection matrix $\Vb\Vb^\top$. 
    This can also be understood by the fact that the Moore-Penrose inverse $S^\dagger$ of a matrix $S$ is continuous if $S$ has fixed rank, and that the projection matrix can be expressed as $\Vb\Vb^\top=\Mb^\dagger \Mb$. Then we just need the continuity of 
    \begin{align}
       \tilde \Vb=\arg\min_{\tilde \Vb:\tilde \Vb\tilde \Vb^\top=\Vb\Vb^\top}\|\Vb_0-\tilde \Vb\|_F^2
    \end{align}
    with respect to the factor $\Vb\Vb^\top$.

    As $\tilde \Vb=\Vb \Qb^\star$ and $\Qb^\star$ can be expressed as in~\eqref{eq:Qoptimal}, the unique solution $\tilde \Vb$ can be expressed as
    \begin{align}
        \tilde \Vb=\Vb\Vb^\top\Vb_0(\Vb_0^\top\Vb\Vb^\top\Vb_0)^{-\frac12}.
    \end{align}
    The continuity of $\tilde \Vb$ with respect to the projection matrix $\Vb\Vb^\top$ is then clear.
\end{proof}
\subsection{Proof of Theorem~\ref{thm:tgphomeomorphism-formal}}\label{proof:tgphomeomorphism}

We first state the formal version of the TGP homeomorphism result (informal version: Thm.~\ref{thm:tgphomeomorphism} in the main text).

\begin{theorem}[TGP Homeomorphism, formal]
\label{thm:tgphomeomorphism-formal}
    For any $l=1,\dots,d$, let $\cI_l=\{l,l+1,\cdots,d\}$. Given anchor matrices $\Ub_{k,0}\in\stie(n_k,r_k)$ for $k\in\cI_l$, let
    \begin{align}
    &\mathcal M(n_1,\dots,n_{l-1},(n_{l},r_{l}\mid\Ub_{l,0}),\dots,(n_d,r_d\mid\Ub_{d,0}))\nonumber\\
        =&\{\mathcal X\in\mathbb R^{n_1\times\cdots\times n_d}:\mathrm{rank}(\sk(\mathcal X))=r_k,\Ub_{k,0}^\top\sk(\mathcal X)\Ub_{k,0}\succ 0,k\in\cI_l\}
    \end{align}
    be the set of tensors whose mode-$k$ unfolding has rank $r_k$ and the mode-$k$ subspace contains no direction that is perpendicular to the anchor matrix $\Ub_{k,0}$, for each $k\in\cI_l$.

    We also abbreviate
    \begin{align}
        \mathcal M(n_1,\dots,n_{l-1},(r_{l})_*,\dots,(r_d)_*)
        :=\mathcal M(n_1,\dots,n_{l-1},(r_{l},r_{l}\mid\Ub_{l,0}),\dots,(r_d,r_d\mid\Ub_{d,0})),
    \end{align}
    which is the set of tensors whose mode-$k$ unfolding has full rank $r_k$.

    Then the following map is a homeomorphism,
    \begin{align}
\Phi(\cdot\mid\Ub_{l,0},\dots,\Ub_{d,0})&:\mathcal M(n_1,\dots,n_{l-1},(n_{l},r_{l}\mid\Ub_{l,0}),\dots,(n_d,r_d\mid\Ub_{d,0}))\nonumber\\
        &\mapsto\mathcal M(n_1,\dots,n_{l-1},(r_{l})_*,\dots,(r_d)_*)\nonumber\\
        &\qquad\otimes\mathrm{gr}(n_{l},r_{l}\mid\Ub_{l,0})\otimes \cdots\otimes\mathrm{gr}(n_{d},r_{d}\mid\Ub_{d,0}),\nonumber\\
        \mathcal X&\mapsto(\mathcal X\times_{l}\tilde \Ub_{l}\cdots\times_d\tilde \Ub_d,\tilde \Ub_{l},\dots,\tilde \Ub_d),
    \end{align}
    where $\tilde \Ub_k= \op(\Ub_k,\Ub_{k,0})$ and $\Ub_k$ consists of the top-$r_k$ singular vectors of $\sk(\cX)$ for $k=l,\dots,d$.
\end{theorem}

\begin{proof}
    From the proof in App.~\ref{proof:matrix-case}, we have already proved the continuity of the map $\Phi(\cdot\mid\Vb_0):\Mb\to(\Mb \tilde \Vb,\tilde \Vb)$. For the tensor case, notice that the map $\Phi(\cdot\mid\Ub_{l,0},\dots,\Ub_{d,0})$ is obtained by the iterative use of $\Phi(\cdot\mid\Ub_{k,0})$ for different unfoldings of the tensor. Thus, the continuity of $\Phi(\cdot\mid\Ub_{l,0},\dots,\Ub_{d,0})$ is also clear.
    
    More precisely, for each $k$ with $l\leq k\leq d$, the map from $\mathcal X\to \tilde \Ub_k$ is given by $\tilde \Ub_k=\op(\Ub_k,\Ub_{k,0})$, with $\Ub_k$ obtained from the SVD of $\mathcal X_k=\Ub_k\bSigma_k \Vb_k^\top$, where $\mathcal X_k\in\mathbb R^{n_k\times(n_1\cdots n_{k-1}n_{k+1}\cdots n_d)}$ is the matricization of $\mathcal X$. Equivalently, we have  $\Phi(\mathcal X_k\mid\Ub_{k,0})=(\Mb_k,\tilde \Ub_k)$ for some $\Mb_k$. Since the matricization $\mathcal X\to\mathcal X_k$ is clearly continuous, the map $\mathcal X\to\tilde\Ub_k$ is also continuous. For bijectivity, since $\tilde\Ub_k$ are all uniquely determined by $\mathcal X$, we have the injectivity of $\Phi(\cdot\mid\Ub_{l,0},\dots,\Ub_{d,0})$. The surjectivity comes from the fact that for any $(\cC,\tilde \Ub_{l},\dots,\tilde \Ub_d)$, if $\mathcal X=\cC\times_l\tilde \Ub_l^\top\cdots\times_d\tilde \Ub_d^\top$, then the projection matrix of the mode-$k$ subspace of $\mathcal X$ is $\tilde \Ub_k\tilde\Ub_k^\top$ which is independent of other Grassmannian factors, so the $k$-th Grassmannian factor of $\mathcal X$ is just $\tilde \Ub_k$.

    The inverse map is still a product map and is also continuous, so the proof is finished.
\end{proof}

\section{Experiment Details} \label{app:exp_settings}
All baseline results share the same number of training epochs, batch size, and sampling steps as DiffATS,
summarized in Tab.~\ref{tab:hyperparams}.
All experiments are run on a single NVIDIA H100 GPU.

\begin{table}[h]
\centering
\caption{\textbf{Training hyperparameters across all datasets and methods.}}
\label{tab:hyperparams}
\begin{tabular}{l ccc}
\toprule
Dataset & Epochs & Batch Size & Sampling Steps \\
\midrule
CelebA-HQ                    & 1000                  & \multirow{6}{*}{32} & \multirow{6}{*}{250} \\
Moving MNIST                 & 2000                  &                     &                      \\
1-d Burgers                  & 1000 &                     &                      \\
1-d reaction-diffusion       & 1000                      &                     &                      \\
2-d Burgers                  & 500  &                     &                      \\
2-d K\'arm\'an vortex street &   500                    &                     &                      \\
\bottomrule
\end{tabular}
\end{table}

\subsection{Compression Ratio}

For each dataset, Tab.~\ref{tab:reconstruction_error} reports the relative $\ell_2$ reconstruction 
error of the rank-$r$ Tucker projection used in DiffATS, 
computed as $\|\cX - \hat\cX\|_F / \|\cX\|_F$ over $100$ test samples. The chosen ranks retain the dominant signal 
content, with mean relative errors ranging from 
$\sim\!10^{-6}$ (1-d reaction-diffusion) to $\sim\!2.2 \times 10^{-1}$ (Moving MNIST). 
As a reference, CelebA-HQ images reconstructed at rank $32$ per $32 \times 32$ 
patch attain a mean PSNR of $34.11 \pm 3.78$ dB (range $[25.46, 41.94]$).

\begin{table}[h]
\centering
\setlength{\tabcolsep}{4pt}
\small
\caption{\textbf{Rank-$r$ reconstruction error.} Relative $\ell_2$ error $\|\cX - \hat\cX\|_F / \|\cX\|_F$ of the Tucker projection used in DiffATS, evaluated over $100$ test samples per dataset. For 1-d datasets, shapes are reported after patchification.}
\label{tab:reconstruction_error}
\resizebox{\textwidth}{!}{%
\begin{tabular}{l l c c c c c}
\toprule
Dataset & Shape & Rank & Mean & Std & Min & Max \\
\midrule
CelebA-HQ 1024 (patch size $32 \times 32$) & $1024{\times}1024{\times}3$ & $32$ & $4.45{\times}10^{-2}$ & $2.19{\times}10^{-2}$ & $1.56{\times}10^{-2}$ & $1.18{\times}10^{-1}$ \\
Moving MNIST & $20{\times}64{\times}64$ & $[15, 64, 20]$ & $2.24{\times}10^{-1}$ & $4.48{\times}10^{-2}$ & $1.15{\times}10^{-1}$ & $3.46{\times}10^{-1}$ \\
2-d Burgers & $200{\times}128{\times}128$ & $[5, 20, 20]$ & $3.26{\times}10^{-3}$ & $2.42{\times}10^{-3}$ & $4.41{\times}10^{-5}$ & $1.18{\times}10^{-2}$ \\
2-d K\'arm\'an vortex  & $200{\times}128{\times}128$ & $[10, 128, 30]$ & $3.40{\times}10^{-2}$ & $3.92{\times}10^{-2}$ & $1.02{\times}10^{-3}$ & $1.78{\times}10^{-1}$ \\
1-d Burgers (patch size $32 \times 20$)& $320{\times}640$ & $32$ & $5.62{\times}10^{-3}$ & $6.73{\times}10^{-3}$ & $3.24{\times}10^{-8}$ & $2.52{\times}10^{-2}$ \\
1-d reaction-diffusion (patch size $32 \times 20$) & $320{\times}640$ & $16$ & $4.66{\times}10^{-6}$ & $1.56{\times}10^{-5}$ & $3.82{\times}10^{-7}$ & $1.55{\times}10^{-4}$ \\
\bottomrule
\end{tabular}%
}
\end{table}

Tab.~\ref{tab:compression_ratios_all} summarizes the per-sample compression ratios used by each method across all datasets. Across all datasets, DiffATS uses a compression ratio at least as large as those of the compressed-domain baselines (DCTdiff, SDIFT, AvgPool); FNO operates directly on the full-resolution field and does not compress.

\begin{table}[h]
\centering
\small
\caption{\textbf{Per-sample compression ratios across methods and datasets.} 
``--'' indicates the baseline is not applied. 
FNO has compression ratio $1$ since it operates on the full-resolution field.}
\label{tab:compression_ratios_all}
\begin{tabular}{l ccccc}
\toprule
Dataset & DiffATS & DCTdiff & SDIFT & AvgPool & FNO \\
\midrule
CelebA-HQ              & $16.00$  & $16.00$  & $16.00$  & $16.00$ & --  \\
Moving MNIST           & $3.94$   & $3.94$   & $3.94$   & $4.00$  & --  \\
1-d Burgers            & $6.67$   & $6.67$   & $6.67$   & $4.00$  & $1.00$ \\
1-d reaction-diffusion & $13.33$  & $13.33$  & $13.33$  & $4.00$  & $1.00$ \\
2-d Burgers            & $210.31$ & $210.31$ & $202.00$ & $64.00$ & $1.00$ \\
2-d K\'arm\'an vortex  & $137.41$ & $130.60$ & $118.87$ & $64.00$ & $1.00$ \\
\bottomrule
\end{tabular}
\end{table}

\subsection{Patchification}\label{app:patchification}
Fig.~\ref{fig:patchification} illustrates the patchification process used in our experiments.
Although not all datasets require it, we apply patchification to CelebA-HQ, 1-d Burgers,
and 1-d reaction-diffusion to reduce the truncation error of the low-rank approximation
and improve generation quality.

For each layer of the tensor, we partition the spatial dimensions into non-overlapping
patches and flatten each patch into a vector. Concatenating these vectors yields a new
tensor on which we apply TGP (or MGP in the matrix case).

Tab.~\ref{tab:svd_patchify_ablation} demonstrates the effectiveness of patchification in preserving information.

\begin{figure}[h]
\centering
\begin{minipage}[b]{0.40\textwidth}
    \centering
    \includegraphics[width=\linewidth]{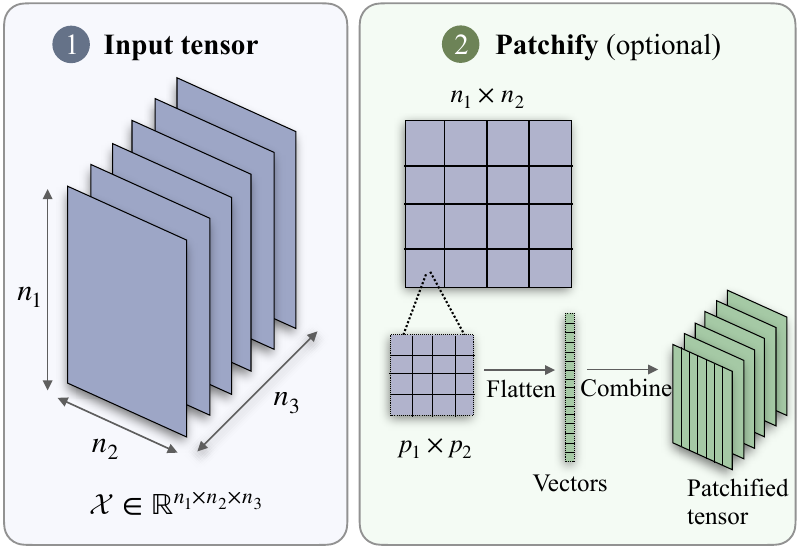}
    \caption{\textbf{Patchification}.}
    \label{fig:patchification}
\end{minipage}\hfill
\begin{minipage}[b]{0.58\textwidth}
    \centering
    \small
    \captionof{table}{\textbf{SVD low-rank reconstruction on CelebA-HQ ($1024\times1024 \times3$, $100$ images, seed $42$).} Patchify uses $32\times32$ patches; the no-patch variant applies SVD directly to the $1024\times1024$ pixel matrix. Metrics are averaged across the three RGB channels and $100$ images.}
    \label{tab:svd_patchify_ablation}
    \begin{tabular}{c cc cc}
    \toprule
    & \multicolumn{2}{c}{RMSE $\downarrow$} & \multicolumn{2}{c}{PSNR $\uparrow$ (dB)} \\
    \cmidrule(lr){2-3} \cmidrule(lr){4-5}
    Rank $r$ & Patchify & No patch & Patchify & No patch \\
    \midrule
    $16$  & $\mathbf{7.55}$ & $12.35$ & $\mathbf{31.03}$ & $26.50$ \\
    $32$  & $\mathbf{5.52}$ & $8.57$  & $\mathbf{33.99}$ & $29.80$ \\
    $64$  & $\mathbf{3.47}$ & $5.54$  & $\mathbf{38.49}$ & $33.83$ \\
    $128$ & $\mathbf{1.74}$ & $2.89$  & $\mathbf{45.19}$ & $40.01$ \\
    \bottomrule
    \end{tabular}
\end{minipage}
\end{figure}

\subsection{DiffATS}
Tab.~\ref{tab:diffats_hyperparams}
summarizes the hyperparameter settings of DiffATS across all datasets.
We use a DiT backbone~\cite{peebles2023scalable} as the score network, with the per-dataset configuration detailed in Tab.~\ref{tab:dit_backbone}.

\begin{table}[h]
\centering
\caption{\textbf{Hyperparameter settings of DiffATS for all datasets.}}
\label{tab:diffats_hyperparams}
\resizebox{\textwidth}{!}{%
\begin{tabular}{l c c c c}
\toprule
Dataset & Patchification & Diffusion Type & Scheduler & Optimizer \\
\midrule
CelebA-HQ              & \checkmark   & Unconditional & \multirow{6}{*}{Linear} & \multirow{6}{*}{AdamW ($\mathrm{lr}=10^{-4}$, $\beta=(0.9, 0.99)$)} \\
Moving MNIST           & $\times$     & Unconditional & & \\
1-d Burgers            & \checkmark   & Conditional   & & \\
1-d reaction-diffusion & \checkmark   & Conditional   & & \\
2-d Burgers            & $\times$     & Conditional   & & \\
2-d K\'arm\'an         & $\times$     & Conditional   & & \\
\bottomrule
\end{tabular}%
}
\end{table}

\begin{table}[h]
\centering
\caption{\textbf{DiT backbone configuration of DiffATS across datasets.}}
\label{tab:dit_backbone}
\begin{tabular}{l c c c}
\toprule
Dataset & Hidden Dim & Depth & Heads \\
\midrule
CelebA-HQ              & \multirow{3}{*}{$768$} & \multirow{3}{*}{$12$} & \multirow{3}{*}{$12$} \\
1-d Burgers            & & & \\
1-d reaction-diffusion & & & \\
\midrule
Moving MNIST           & \multirow{3}{*}{$512$} & \multirow{3}{*}{$8$}  & \multirow{3}{*}{$8$} \\
2-d Burgers            & & & \\
2-d K\'arm\'an         & & & \\
\bottomrule
\end{tabular}
\end{table}

\subsection{Diffusion Training and Sampling}
All diffusion-based methods (DiffATS, DCTdiff, SDIFT, and AvgPool) share the same training and sampling protocol; only the representation on which the score network operates differs.
For training, we adopt the standard DDPM formulation with $T=1000$ timesteps and a linear $\beta$-schedule, optimizing the $\epsilon$-prediction loss.
For sampling, we use DDIM~\cite{ddim} with $250$ steps. DDIM defines a family of non-Markovian forward processes that share the same marginals as DDPM, which admits a deterministic reverse sampler and yields high-quality samples in far fewer steps than the $T=1000$ used at training time. FNO is a deterministic neural operator and does not involve diffusion training or sampling.

\subsection{Baselines}
\noindent\textbf{DCTdiff.}
For DCTdiff, the original implementation operates in the YCbCr color space and uses chroma subsampling. On CelebA-HQ, to provide a fair comparison baseline, we train DCTdiff entirely in RGB by applying a block DCT independently to each RGB channel without chroma subsampling. To adapt DCTdiff to spatiotemporal tensor data such as 2-d Burgers and 2-d K\'arm\'an vortex, we extend the 2-d block DCT to a 3-d block DCT. The $(T,H,W)$ volume is divided into non-overlapping rectangular prisms; then 3-d DCT is applied within each block, and coefficients are ordered using a 3-d zigzag traversal from low to high frequency. The remaining compression and diffusion steps then follow the DCTdiff procedure in the resulting compressed DCT space.

\noindent\textbf{SDIFT.}
For SDIFT, each trajectory is represented as an explicit sequence of Tucker cores. Equivalently, the compressed latent for one trajectory has shape $(T \times R_1 \times R_2 \times R_3)$, where $T$ indexes time, and $(R_1,R_2,R_3)$ denote the spatial Tucker-core ranks at each timestep. Importantly, SDIFT does not compress across time, but instead keeps a separate Tucker core for each timestep. For the 1-d PDE datasets, such as Burgers and reaction-diffusion with shape $(T,X)$, we use a sequence of one-dimensional spatial cores, represented as $(T \times 1 \times 1 \times R_x)$, to fit the Tucker-core representation. For 2-d Burgers and 2-d K\'arm\'an vortex, with $T=201$ and $H=W=128$, this means the compressed representation still contains $T$ cores per trajectory. Therefore under a fixed compression budget, long temporal sequences require smaller spatial/core ranks, which can reduce reconstruction quality and represents a limitation of SDIFT for PDEs with many timesteps.

\noindent\textbf{FNO.}
FNO, by contrast, operates directly on the full-resolution field without spatial compression. We train FNO on one-step prediction pairs $(u_t,u_{t+1})$ and generate the full temporal sequence autoregressively. Since each epoch iterates over $N(T-1)$ full-resolution training pairs, the per-epoch cost scales with the full spatiotemporal data volume rather than with a compressed latent representation. Consequently, under similar time budgets, FNO requires substantially higher training cost per epoch than the compressed-domain baselines. In our experiments, this higher computational cost did not translate into stronger predictive performance.

\noindent\textbf{Average pooling.}
Average pooling serves as a straightforward spatial compression baseline that reduces each field to a compact latent representation without any learned transform. Concretely, for a 2-d PDE trajectory with shape $(T, H, W)$, we apply average pooling with a spatial stride of $k$, yielding a compressed representation of shape $(T, H/k, W/k)$. For 1-d PDE datasets with field shape $(T, X)$, we apply 1-d average pooling along the spatial axis to obtain $(T, X/k)$. The compression ratio is thus $k^2$ for 2-d fields and $k$ for 1-d fields. For the CelebA-HQ dataset, which consists of static RGB images of shape $(3,H,W)$, we apply 2-d average pooling and give compressed representations of shape $(3,H/k,W/k)$. For Moving MNIST where each sample is a grayscale video of shape $(T,H,W)$, we apply 2-d average pooling to each frame, yielding $(T,H/k,W/k)$ with a $k^2$ compression ratio. At inference, generated samples are upsampled back to the original resolution via linear interpolation for 1-d fields and bilinear interpolation for all 2-d domains, and all metrics are evaluated at the original resolution.

\subsection{Network Capacity}
To ensure a fair comparison, we configure each baseline so that its score network has a number of parameters comparable to DiffATS. Tab.~\ref{tab:param_counts} reports the per-dataset parameter counts, and differences across methods are within a few percent on every dataset.

\begin{table}[h]
\centering
\setlength{\tabcolsep}{6pt}
\small
\caption{\textbf{Number of score-network parameters across methods and datasets.} ``M'' denotes million. ``--'' indicates the baseline is not applied (FNO is PDE-specific).}
\label{tab:param_counts}
\begin{tabular}{l c c c c}
\toprule
Dataset & DiffATS & DCTdiff & FNO & SDIFT \\
\midrule
CelebA-HQ                    & $134.40$M & $134.00$M & --        & $128.00$M \\
Moving MNIST                 & $38.90$M  & $39.70$M  & --        & $40.27$M  \\
1-d Burgers                  & $131.93$M & $132.10$M & $131.78$M & $132.20$M \\
1-d reaction-diffusion       & $131.93$M & $132.10$M & $131.78$M & $133.30$M \\
2-d Burgers                  & $40.11$M  & $41.68$M  & $40.22$M  & $40.00$M  \\
2-d K\'arm\'an vortex street & $41.33$M  & $42.40$M  & $41.50$M  & $42.40$M  \\
\bottomrule
\end{tabular}
\end{table}

\subsection{PDE Dataset Generation Details} \label{app:pde_gen_details}
\noindent\textbf{1-d Burgers' equation.}
For the dataset generation, we follow the same procedure as in \cite{takamoto2022pdebench}.
The 1-d Burgers' equation is a nonlinear PDE modeling the diffusion process in fluid dynamics
as 
\begin{align*}
\partial_t u(t,x) + \partial_x\bigl(u^2(t,x)/2\bigr)
&= \nu/\pi \partial_{xx} u(t,x),
\quad x \in (0,1),\ t \in (0,2], \\
u(0,x) &= u_0(x),
\quad x \in (0,1),
\end{align*}
with periodic boundary conditions $u(t, 0) = u(t, 1)$ and $\partial_x u(t, 0) = \partial_x u(t, 1)$ for $t \in (0, 2]$, where $\nu$ is the diffusion coefficient. In our datasets, we generate training samples and test samples
with $\nu$ i.i.d. sampled from $[10^{-5},\ 5\times10^{-5},\ 10^{-4},\ 5\times10^{-4},\ 10^{-3},\ 5\times10^{-3},\ 10^{-2},\ 5\times10^{-2},\ 10^{-1}]$.
And $\nu$ is also provided as a conditioning variable to the diffusion model.

\noindent\textbf{1-d reaction-diffusion.}
For 1-d reaction-diffusion dataset generation, we also follow the same procedure as in \cite{takamoto2022pdebench}.
The equation is expressed as
\begin{align*}
\partial_t u(t,x) - \nu \partial_{xx} u(t,x) - \rho u(1-u) &= 0,
\quad x \in (0,1),\ t \in (0,1], \\
u(0,x) &= u_0(x),
\quad x \in (0,1),
\end{align*}
with periodic boundary conditions $u(t, 0) = u(t, 1)$ and $\partial_x u(t, 0) = \partial_x u(t, 1)$ for $t \in (0, 1]$, where $\nu$ is the diffusion coefficient and $\rho$ is the reaction rate of the logistic source term.
For both training and test samples, we draw $\nu$ i.i.d. from
$[10^{-5},\ 5\times10^{-5},\ 10^{-4},\ 5\times10^{-4},\ 10^{-3},\ 5\times10^{-3},\ 10^{-2},\ 5\times10^{-2},\ 10^{-1}]$
and $\rho$ i.i.d. from $[0.1,\ 0.5,\ 1.0,\ 2.0]$.
Both $\nu$ and $\rho$ are provided as conditioning variables injected into the diffusion model.

\noindent\textbf{2-d Burgers' equation.}
The dataset generation process follows the procedure in~\cite{koehler2024apebench}.
The 2-d Burgers' equation describes the evolution of a two-dimensional
velocity field \(\mathbf{u}=(u_1,u_2)\) on a periodic domain:
\[
\begin{aligned}
\partial_t \mathbf{u} + (\mathbf{u}\cdot \nabla)\mathbf{u}
&= \nu \Delta \mathbf{u},
&& \mathbf{x}\in (0,1)^2,\quad t\in (0,T],\\
\mathbf{u}(0,\mathbf{x})
&= \mathbf{u}_0(\mathbf{x}),
&& \mathbf{x}\in (0,1)^2,
\end{aligned}
\]
with periodic boundary conditions $\mathbf{u}(t, 0, y) = \mathbf{u}(t, 1, y)$ 
and $\mathbf{u}(t, x, 0) = \mathbf{u}(t, x, 1)$ for $t \in (0, T]$. Here $\mathbf{x} = (x, y)$ and $\nu > 0$ is the viscosity coefficient; the vector equation corresponds to two coupled nonlinear PDEs for the velocity components $u_1$ and $u_2$.
For both training and test samples, $\nu$ is sampled i.i.d. from the continuous
interval $\mathrm{Uniform}(1\times 10^{-5},1\times 10^{-4})$ and provided as a
conditioning variable to the diffusion model.
For training, the two components $u_1$ and $u_2$ are pooled as independent samples
rather than modeled as a joint vector field.

\noindent\textbf{2-d K\'arm\'an vortex street.}
The K\'arm\'an vortex street describes the periodic vortex shedding behind a
2-d incompressible flow past a cylinder at moderate Reynolds number.
We generate trajectories with a 2-d Lattice Boltzmann Method (LBM) solver
implemented in Taichi, following the setup of \url{https://github.com/hietwll/LBM_Taichi#lbm_taichi}.
The simulation domain is a rectangular channel with the following boundary conditions: Dirichlet inflow $\mathbf{u} = (0.1, 0)$ on the left, zero-gradient (Neumann) outflow on the right, Dirichlet no-slip ($\mathbf{u} = 0$) on the top and bottom walls, and no-slip (bounce-back) on the cylinder surface.
The dataset enumerates the Cartesian product of four parameters, yielding
$5\times 5\times 4\times 2 = 200$ trajectories in total:
the LBM viscosity $\nu \in \{0.020,\ 0.022,\ 0.025,\ 0.028,\ 0.032\}$,
the cylinder center $c_x \in \{26,\ 29,\ 32,\ 35,\ 38\}$ and
$c_y \in \{56,\ 61,\ 67,\ 72\}$ (in pixels), and the cylinder radius
$r \in \{7,\ 9\}$ (in pixels). All four parameters are provided as conditioning variables to the diffusion model. The resulting Reynolds numbers $\mathrm{Re} = 0.1 \cdot 2r / \nu$ fall in $[50, 80]$, the classical periodic vortex-shedding regime.

\section{Additional Experimental Results}\label{app:additional_results}
In this section, we provide additional experimental results and generated samples for reference.

\subsection{Full PDE Metrics}\label{app:pde_full_metrics}
We report the complete metric breakdown for the PDE benchmarks summarized in Tab.~\ref{tab:pde_main}. Each entry is the mean $\pm$ standard deviation over $5$ random seeds. $\mathrm{Rel.~Err.}_1$ and $\mathrm{Rel.~Err.}_2$ are relative errors in $\ell_1$ and $\ell_2$ norms, respectively; $\mathrm{Avg.~rMSE}$ is the average rMSE across time frames.

\begin{table}[h]
\centering
\setlength{\tabcolsep}{3pt}
\small
\caption{\textbf{Full results on 1-d PDE datasets.} Lower is better for all metrics.}
\label{tab:pde_1d_full}
\resizebox{\textwidth}{!}{%
\begin{tabular}{l ccc ccc}
\toprule
& \multicolumn{3}{c}{1-d Burgers} & \multicolumn{3}{c}{1-d reaction-diffusion} \\
\cmidrule(lr){2-4} \cmidrule(lr){5-7}
Method
& $\mathrm{Rel.~Err.}_1\!\downarrow$ & $\mathrm{Rel.~Err.}_2\!\downarrow$ & $\mathrm{Avg.~rMSE}\!\downarrow$
& $\mathrm{Rel.~Err.}_1\!\downarrow$ & $\mathrm{Rel.~Err.}_2\!\downarrow$ & $\mathrm{Avg.~rMSE}\!\downarrow$ \\
\midrule
DCTdiff~\cite{ning2024dctdiff}    & $0.2103 \pm 0.0090$           & $0.2953 \pm 0.0100$           & $0.1195 \pm 0.0040$                      & $0.0681 \pm 0.0002$           & $0.0713 \pm 0.0001$           & $0.0497 \pm 0.0001$ \\
SDIFT~\cite{chen2025generating}   & $0.3513 \pm 0.3100$           & $0.3884 \pm 0.2600$           & $0.1497 \pm 0.0590$                      & $0.3355 \pm 0.0087$           & $0.3406 \pm 0.0069$           & $0.2199 \pm 0.0030$ \\
FNO~\cite{li2020fourier}          & $0.1296 \pm 0.1500$           & $0.1913 \pm 0.1800$           & $0.0697 \pm 0.1100$                      & $0.2649 \pm 0.0061$           & $0.3114 \pm 0.0069$           & $0.2103 \pm 0.0045$ \\
AvgPool                           & $0.5933 \pm 0.0140$           & $0.6602 \pm 0.0120$           & $0.2465 \pm 0.0019$                      & $0.3405 \pm 0.0038$           & $0.3975 \pm 0.0038$           & $0.2655 \pm 0.0025$ \\
DiffATS (ours)                    & $\mathbf{0.0816 \pm 0.0001}$  & $\mathbf{0.1203 \pm 0.0001}$  & $\mathbf{0.0517 \pm 0.0000}$ & $\mathbf{0.0076 \pm 0.0000}$ & $\mathbf{0.0085 \pm 0.0000}$ & $\mathbf{0.0057 \pm 0.0000}$ \\
\bottomrule
\end{tabular}%
}
\end{table}

\begin{table}[h]
\centering
\setlength{\tabcolsep}{3pt}
\small
\caption{\textbf{Full results on 2-d PDE datasets.} Lower is better for all metrics.}
\label{tab:pde_2d_full}
\resizebox{\textwidth}{!}{%
\begin{tabular}{l ccc ccc}
\toprule
& \multicolumn{3}{c}{2-d Burgers} & \multicolumn{3}{c}{2-d K\'arm\'an vortex street} \\
\cmidrule(lr){2-4} \cmidrule(lr){5-7}
Method
& $\mathrm{Rel.~Err.}_1\!\downarrow$ & $\mathrm{Rel.~Err.}_2\!\downarrow$ & $\mathrm{Avg.~rMSE}\!\downarrow$
& $\mathrm{Rel.~Err.}_1\!\downarrow$ & $\mathrm{Rel.~Err.}_2\!\downarrow$ & $\mathrm{Avg.~rMSE}\!\downarrow$ \\
\midrule
DCTdiff~\cite{ning2024dctdiff}    & $0.3036 \pm 0.0046$          & $0.3241 \pm 0.0048$          & $0.0639 \pm 0.0011$          & $0.3789 \pm 0.0540$          & $0.5550 \pm 0.0710$          & $0.0040 \pm 0.0006$ \\
SDIFT~\cite{chen2025generating}   & $0.7831 \pm 0.2370$          & $0.7323 \pm 0.2030$          & $0.1070 \pm 0.0015$           & $0.8053 \pm 0.1200$          & $0.8654 \pm 0.1200$          & $0.7631 \pm 0.2000$ \\
FNO~\cite{li2020fourier}          & $0.2971 \pm 0.2000$          & $0.3423 \pm 0.3000$          & $0.0790 \pm 0.1000$          & $0.2508 \pm 0.0200$          & $0.3442 \pm 0.0190$          & $0.0025 \pm 0.0002$ \\
AvgPool                           & $0.3544 \pm 0.0036$          & $0.3747 \pm 0.0029$          & $0.0780 \pm 0.0006$          & $0.6468 \pm 0.0049$          & $0.6161 \pm 0.0020$          & $0.0044 \pm 0.0000$ \\
DiffATS (ours)                    & $\mathbf{0.1775 \pm 0.0005}$ & $\mathbf{0.1830 \pm 0.0005}$ & $\mathbf{0.0424 \pm 0.0001}$ & $\mathbf{0.1640 \pm 0.0003}$ & $\mathbf{0.2233 \pm 0.0005}$ & $\mathbf{0.0016 \pm 0.0000}$ \\
\bottomrule
\end{tabular}%
}
\end{table}

\subsection{Physical Diagnostics on 2-d Burgers}
\label{app:pde_physics}

Beyond the relative $\ell_2$ error in Tab.~\ref{tab:pde_main}, we evaluate
whether generated trajectories respect the physics of the governing PDE.
We compute four diagnostics on all $500$ test trajectories of 2-d Burgers:
energy dissipation, one-step PDE residual, periodic boundary consistency,
and long-horizon stability.
Throughout, $\hat{\cX} \in \RR^{n_1 \times n_2 \times T}$ denotes a generated
trajectory with frames $\hat{\Xb}_1, \dots, \hat{\Xb}_T$, $\cX$ denotes the
ground-truth trajectory with frames $\Xb_1, \dots, \Xb_T$, and $T = 200$.

\noindent\textbf{Energy dissipation.}
Let $E_t = \frac{1}{2 n_1 n_2} \sum_{i,j} \|\mathbf{u}_t(x_i, y_j)\|_2^2$
denote the kinetic energy of the ground-truth velocity field at frame $t$,
averaged over the spatial grid, and let $\hat{E}_t$ be the same quantity
for the generated frame.
The energy error is the relative error of the kinetic-energy trajectory
over the prediction horizon,
\begin{equation*}
e_E = \sqrt{\frac{\sum_{t=1}^{T} (\hat{E}_t - E_t)^2}{\sum_{t=1}^{T} E_t^2}} .
\end{equation*}

\noindent\textbf{One-step PDE residual.}
Let $S$ denote the one-step update of the data-generating
solver~\cite{koehler2024apebench}.
The one-step residual
\begin{equation*}
\rho_{\mathrm{step}}
= \frac{\|S(\hat{\Xb}_t) - \hat{\Xb}_{t+1}\|_F}{\|\hat{\Xb}_{t+1} - \hat{\Xb}_t\|_F}
\end{equation*}
measures the consistency of consecutive generated frames with the governing
PDE and its discretization, normalized by the motion over one step.
We average it over frame pairs.

\noindent\textbf{Periodic boundary consistency.}
The boundary-consistency residual along $x$ is
\begin{equation*}
r_x = \frac{\|\hat{\cX}_{1,:,:} - \hat{\cX}_{n_1,:,:}\|_F}{\|\hat{\cX}\|_F + \epsilon},
\end{equation*}
where $\hat{\cX}_{i,:,:}$ is the $i$-th slice of $\hat{\cX}$ along $x$ and
$\epsilon$ is a small constant.
The residual $r_y$ is defined analogously along $y$.
Since the grid is periodic, the first and last slices are adjacent cells,
and smaller values indicate better consistency with the periodic boundary
conditions.

\noindent\textbf{Long-horizon stability.}
The final-frame error
$\varepsilon_{200} = \|\hat{\Xb}_{200} - \Xb_{200}\|_F / \|\Xb_{200}\|_F$
is the relative error of the last generated frame against the ground-truth
frame, which reflects accuracy at the end of the prediction horizon.

\begin{table}[h]
\centering
\small
\setlength{\tabcolsep}{5pt}
\caption{
\textbf{Physical diagnostics on 2-d Burgers.}
Energy error $e_E$, one-step PDE residual $\rho_{\mathrm{step}}$,
boundary-consistency residuals $r_x$ and $r_y$, and final-frame error
$\varepsilon_{200}$, averaged over the $500$ test trajectories.
Lower is better for all metrics.
}
\label{tab:pde_physics}
\begin{tabular}{@{}lcccc@{}}
\toprule
Method & $e_E \downarrow$ & $\rho_{\mathrm{step}} \downarrow$ & $r_x \,/\, r_y \downarrow$ & $\varepsilon_{200} \downarrow$ \\
\midrule
DCTdiff~\cite{ning2024dctdiff}  & $0.2407$ & $1.5780$ & $0.01088$ / $0.01190$ & $0.475$ \\
SDIFT~\cite{chen2025generating} & $0.9405$ & $2.0896$ & $0.02736$ / $0.04398$ & $1.278$ \\
FNO~\cite{li2020fourier}        & $0.4185$ & $2.3688$ & $0.00738$ / $0.00769$ & $0.816$ \\
AvgPool                         & $0.3452$ & $1.5687$ & $0.04221$ / $0.04786$ & $0.553$ \\
DiffATS (ours)                  & $\mathbf{0.0379}$ & $\mathbf{1.3384}$ & $\mathbf{0.00565}$ / $\mathbf{0.00634}$ & $\mathbf{0.268}$ \\
\bottomrule
\end{tabular}
\end{table}

\noindent\textbf{Results.}
DiffATS performs best on all four diagnostics (Tab.~\ref{tab:pde_physics}).
Its energy error is approximately $6$ to $25\times$ smaller than those of
the baselines, its one-step residual is the lowest, and its boundary
residuals are below those of every baseline.
The final-frame error of DiffATS is $0.268$, compared with $0.475$ for the
best baseline, which indicates that the generated trajectories remain
accurate at the end of the $200$-frame horizon.

\subsection{OP-Alignment Ablation on PDE Datasets}
\label{app:op_ablation_pde}

The CelebA-HQ ablation in Tab.~\ref{tab:celeba} shows that OP alignment
is essential for image generation.
We repeat this ablation on 2-d Burgers and 1-d reaction-diffusion.
``Without OP'' trains the diffusion model on unaligned data-dependent
Tucker factors, matching the ``Data-dep.\ bases w/o align.'' setting of
Tab.~\ref{tab:celeba}, and all other components are unchanged.
We report the relative $\ell_2$ error of Tab.~\ref{tab:pde_main} together
with the physical diagnostics of App.~\ref{app:pde_physics}.
For 1-d reaction-diffusion, the diagnostics are computed on the scalar
field $u$ over the 1-d periodic grid: the energy of frame $t$ is
$E_t = \frac{1}{2 n_1} \sum_{i} u_t(x_i)^2$, $S$ is the one-step update of
the reaction-diffusion solver, and only the boundary residual $r_x$ is
defined.
Each entry averages over the $500$ test trajectories and reports the
mean $\pm$ standard deviation over $5$ sampling seeds.

\begin{table}[h]
\centering
\small
\setlength{\tabcolsep}{4pt}
\caption{
\textbf{OP-alignment ablation on PDE datasets.}
``Without OP'' trains on unaligned data-dependent Tucker factors, and
``With OP'' is the full DiffATS.
Each entry is the mean $\pm$ standard deviation over $5$ sampling seeds,
averaged over the $500$ test trajectories.
For 1-d reaction-diffusion, only $r_x$ is defined.
Lower is better for all metrics.
}
\label{tab:op_ablation_pde}
\resizebox{\textwidth}{!}{%
\begin{tabular}{@{}lccccc@{}}
\toprule
Method & Rel.\ $\ell_2$ error $\downarrow$ & $e_E \downarrow$ & $\rho_{\mathrm{step}} \downarrow$ & $r_x \,/\, r_y \downarrow$ & $\varepsilon_{200} \downarrow$ \\
\midrule
\multicolumn{6}{@{}l}{\emph{2-d Burgers}} \\
Without OP
& $0.8034 \pm 0.0053$ & $0.0936 \pm 0.0026$ & $1.6911 \pm 0.0235$ & $0.00602 \pm 0.00002$ / $0.00643 \pm 0.00003$ & $0.816 \pm 0.011$ \\
With OP (DiffATS)
& $\mathbf{0.1830 \pm 0.0005}$ & $\mathbf{0.0379 \pm 0.0006}$ & $\mathbf{1.3384 \pm 0.0166}$ & $\mathbf{0.00565 \pm 0.00002}$ / $\mathbf{0.00634 \pm 0.00003}$ & $\mathbf{0.268 \pm 0.001}$ \\
\midrule
\multicolumn{6}{@{}l}{\emph{1-d reaction-diffusion}} \\
Without OP
& $0.0299 \pm 0.0005$ & $0.0299 \pm 0.0005$ & $1.5245 \pm 0.2007$ & $0.00115 \pm 0.00004$ & $0.0254 \pm 0.0004$ \\
With OP (DiffATS)
& $\mathbf{0.0085 \pm 0.0000}$ & $\mathbf{0.0132 \pm 0.0000}$ & $\mathbf{0.7823 \pm 0.0003}$ & $\mathbf{0.00025 \pm 0.00000}$ & $\mathbf{0.0067 \pm 0.0000}$ \\
\bottomrule
\end{tabular}%
}
\end{table}

\noindent\textbf{Results.}
Removing OP alignment increases the relative $\ell_2$ error by
approximately $4.4\times$ on 2-d Burgers and $3.5\times$ on
1-d reaction-diffusion (Tab.~\ref{tab:op_ablation_pde}).
Alignment also improves every physical diagnostic on both benchmarks:
the energy error, the one-step residual, the boundary residuals, and the
final-frame error are all lower with OP.
These results are consistent with the CelebA-HQ ablation in
Tab.~\ref{tab:celeba}, where removing alignment raises the FID from
$23.57$ to $76.94$.

\subsection{Anchor-Overlap Condition and Anchor Selection}
\label{app:anchor_overlap}

Thm.~\ref{thm:matrix-case} and Thm.~\ref{thm:tgphomeomorphism} require
the anchor-overlap condition: every factor subspace has no direction
perpendicular to the anchor subspace, i.e., all principal angles between
$\mathrm{span}(\Vb_i)$ and $\mathrm{span}(\Vb_0)$ are strictly smaller
than $\pi/2$.
We verify this condition on the data.
For each aligned mode, we compute the principal angles
$\theta_1, \dots, \theta_r$ between the factor subspace of every training
and test sample and its medoid anchor.
We summarize them by the mean squared cosine
$\frac{1}{r}\sum_{k=1}^{r} \cos^2\theta_k = \|\Vb_i^\top \Vb_0\|_F^2 / r$,
which lies in $[0, 1]$ with larger values indicating closer subspaces,
together with the median and the maximum of the principal angles over
all samples.

\begin{table}[h]
\centering
\small
\setlength{\tabcolsep}{3pt}
\caption{
\textbf{Principal angles between factor subspaces and the medoid anchor.}
Statistics are computed over all training and test samples of each dataset.
Mean squared cosine lies in $[0, 1]$, with larger values indicating closer
subspaces.
The anchor-overlap condition requires a maximum angle below $90^\circ$.
}
\label{tab:anchor_overlap}
\begin{tabular}{@{}llcccc@{}}
\toprule
Dataset & Mode & Rank / ambient dim. & Mean squared cosine & Median angle & Maximum angle \\
\midrule
2-d Burgers            & Time   & $5 / 200$  & $0.954$ & $0.035^\circ$ & $23.2^\circ$ \\
2-d Burgers            & Height & $20 / 128$ & $0.951$ & $0.087^\circ$ & $62.0^\circ$ \\
2-d Burgers            & Width  & $20 / 128$ & $0.949$ & $0.065^\circ$ & $63.5^\circ$ \\
1-d reaction-diffusion & Patch  & $16 / 640$ & $0.746$ & $3.26^\circ$  & $89.6^\circ$ \\
\bottomrule
\end{tabular}
\end{table}

\noindent\textbf{Results.}
Every observed maximum angle is below $90^\circ$, so the anchor-overlap
condition of Thm.~\ref{thm:matrix-case} and Thm.~\ref{thm:tgphomeomorphism}
holds for every tested mode on both training and test samples
(Tab.~\ref{tab:anchor_overlap}).
The temporal and spatial modes of 2-d Burgers are tightly clustered around
their anchors, with mean squared cosine above $0.94$ and median angles
below $0.1^\circ$, while the patch mode of 1-d reaction-diffusion is more
dispersed.

\noindent\textbf{Anchor selection.}
To test whether the medoid criterion of Sec.~\ref{subsec:mgp} matters, we
replace the medoid anchor of 1-d reaction-diffusion with the least
representative candidate from the same pool, i.e., the anti-medoid that
minimizes $\sum_{j=1}^N \|\Vb_i^\top \Vb_j\|_F^2$, and rebuild the
primitives and retrain the diffusion model with all other settings
unchanged.
Tab.~\ref{tab:anchor_selection} reports the mean squared cosine between
the training and test subspaces and each anchor, together with the
relative $\ell_2$ error and the physical diagnostics of
App.~\ref{app:pde_physics}.

\begin{table}[h]
\centering
\small
\setlength{\tabcolsep}{4pt}
\caption{
\textbf{Medoid versus anti-medoid anchor on 1-d reaction-diffusion.}
The anti-medoid is the least representative candidate in the same pool.
Mean squared cosine is averaged over all training and test subspaces.
Generation metrics are the mean $\pm$ standard deviation over sampling
seeds, averaged over the $500$ test trajectories.
}
\label{tab:anchor_selection}
\resizebox{\textwidth}{!}{%
\begin{tabular}{@{}lcccccc@{}}
\toprule
Anchor & Mean squared cosine $\uparrow$ & Rel.\ $\ell_2$ error $\downarrow$ & $e_E \downarrow$ & $\rho_{\mathrm{step}} \downarrow$ & $r_x \downarrow$ & $\varepsilon_{200} \downarrow$ \\
\midrule
Anti-medoid
& $0.503$ & $0.006890 \pm 0.000001$ & $0.0073 \pm 0.0000$ & $0.7722 \pm 0.0000$ & $0.00029 \pm 0.00000$ & $0.0063 \pm 0.0000$ \\
Medoid (DiffATS)
& $\mathbf{0.746}$ & $\mathbf{0.005632 \pm 0.000010}$ & $\mathbf{0.0055 \pm 0.0000}$ & $\mathbf{0.7114 \pm 0.0002}$ & $\mathbf{0.00028 \pm 0.00000}$ & $\mathbf{0.0057 \pm 0.0000}$ \\
\bottomrule
\end{tabular}%
}
\end{table}

The anti-medoid reduces the mean squared cosine from $0.746$ to $0.503$
and worsens every reported metric: the relative $\ell_2$ error rises from
$0.00563$ to $0.00689$, and the energy error, one-step residual, boundary
residual, and final-frame error all increase.
The anti-medoid still outperforms every baseline in Tab.~\ref{tab:pde_main}.
For independent uniformly random $r$-dimensional subspaces of $\RR^n$,
the expected mean squared cosine is $r/n = 16/640 = 0.025$, so the
anti-medoid overlap of $0.503$ remains about $20\times$ this value, which
may explain its continued effectiveness.
A more representative anchor further improves accuracy.

\subsection{Rank Sensitivity}
\label{app:rank_sensitivity}

The ranks used in the main experiments are listed in
Tab.~\ref{tab:reconstruction_error}.
To assess sensitivity to this choice, we retrain DiffATS from scratch at
one lower and one higher rank on 1-d reaction-diffusion and 2-d Burgers,
keeping every other setting fixed, and regenerate the test trajectories.
Tab.~\ref{tab:rank_sensitivity} reports the relative $\ell_2$
reconstruction error of the low-rank decomposition at each rank, the
relative $\ell_2$ error of the generated trajectories, the energy error
$e_E$ and final-frame error $\varepsilon_{200}$ of
App.~\ref{app:pde_physics}, and the ratio of the relative $\ell_2$ error
to that of the default rank.

\begin{table}[h]
\centering
\small
\setlength{\tabcolsep}{5pt}
\caption{
\textbf{Rank sensitivity on 1-d reaction-diffusion and 2-d Burgers.}
Each row retrains DiffATS from scratch at the given rank with all other
settings fixed.
Rec.\ error is the relative $\ell_2$ reconstruction error of the low-rank
decomposition.
Generation metrics are the mean $\pm$ standard deviation over $5$ sampling
seeds.
The last column is the ratio of the relative $\ell_2$ error to that of the
default rank.
}
\label{tab:rank_sensitivity}
\resizebox{\textwidth}{!}{%
\begin{tabular}{@{}lccccc@{}}
\toprule
Rank & Rec.\ error & Rel.\ $\ell_2$ error $\downarrow$ & $e_E \downarrow$ & $\varepsilon_{200} \downarrow$ & Ratio to default \\
\midrule
\multicolumn{6}{@{}l}{\emph{1-d reaction-diffusion}} \\
$8$             & $7.5{\times}10^{-5}$ & $0.00612 \pm 0.00001$ & $0.00615 \pm 0.00000$ & $0.00574 \pm 0.00000$ & $0.82\times$ \\
$16$ (default)  & $5.4{\times}10^{-6}$ & $0.00747 \pm 0.00000$ & $0.00800 \pm 0.00000$ & $0.00667 \pm 0.00000$ & $1\times$ \\
$32$            & $1.5{\times}10^{-6}$ & $0.01503 \pm 0.00001$ & $0.01732 \pm 0.00002$ & $0.01328 \pm 0.00002$ & $2.01\times$ \\
\midrule
\multicolumn{6}{@{}l}{\emph{2-d Burgers}} \\
$[3, 10, 10]$            & $4.4{\times}10^{-2}$ & $0.2432 \pm 0.0012$ & $0.0554 \pm 0.0006$ & $0.3048 \pm 0.0019$ & $1.32\times$ \\
$[5, 20, 20]$ (default)  & $2.4{\times}10^{-3}$ & $0.1845 \pm 0.0010$ & $0.0389 \pm 0.0008$ & $0.2698 \pm 0.0015$ & $1\times$ \\
$[7, 30, 30]$            & $3.0{\times}10^{-4}$ & $0.1934 \pm 0.0009$ & $0.0403 \pm 0.0004$ & $0.2920 \pm 0.0017$ & $1.05\times$ \\
\bottomrule
\end{tabular}%
}
\end{table}

\noindent\textbf{Results.}
DiffATS remains effective across the tested ranks, and every configuration
outperforms all baselines in Tab.~\ref{tab:pde_main}.
On 2-d Burgers, the lower rank $[3, 10, 10]$ raises the relative $\ell_2$
error by $1.32\times$, whereas the higher rank $[7, 30, 30]$ changes it by
only $1.05\times$.
On 1-d reaction-diffusion, rank $8$ attains a lower generation error than
the default despite a larger reconstruction error, while rank $32$ doubles
the generation error.
When the reconstruction error is already far below the generation error,
a lower rank yields a smaller representation that is easier to model and
faster to train.
We therefore recommend sweeping candidate ranks and choosing the smallest
rank whose reconstruction error is acceptable.

\subsection{Preprocessing Cost Comparison}
\label{app:ae_comparison}

\noindent\textbf{Experimental setup.}
We compare the cost of constructing TGP representations for DiffATS
with the cost of pretraining a 21.3M-parameter 3-d convolutional
autoencoder on the 2-d Burgers dataset.
The training set contains $10{,}000$ trajectories, each with
$201 \times 128 \times 128$ entries per field.
DiffATS preprocessing consists of anchor search, HOOI, and OP alignment,
and does not require training a neural network.
We run this preprocessing on 16 CPU cores without a GPU.
The autoencoder is pretrained on a single H100 GPU, with training
terminated when twelve consecutive validation checks fail to improve
the best validation error.

\begin{table}[h]
\centering
\small
\setlength{\tabcolsep}{4pt}
\caption{
\textbf{Preprocessing cost on 2-d Burgers.}
DiffATS runs on CPU only, while autoencoder pretraining uses an H100 GPU.
Anchor search is timed on a random subset of $1{,}000$ training trajectories;
HOOI and OP alignment are applied to all $10{,}000$ training trajectories.
Host RAM and GPU VRAM are reported separately.
}
\label{tab:preprocessing_cost}
\begin{tabular}{@{}lcccc@{}}
\toprule
Stage & Wall-clock & Compute resources & Host RAM & Peak VRAM \\
\midrule
DiffATS: anchor search
& 107 s & 16 CPU cores & 17.1 GB & -- \\
DiffATS: HOOI + OP alignment
& $\sim 38$ min & 16 CPU cores & 17.1 GB & -- \\
\textbf{DiffATS: total}
& $\boldsymbol{\sim 40}$ \textbf{min}
& 16 CPU cores & 17.1 GB & -- \\
\midrule
Autoencoder pretraining
& 4.5 h & $1 \times$ H100 GPU & 27.2 GB & 52.0 GB \\
\bottomrule
\end{tabular}
\end{table}

\noindent\textbf{Runtime and resource comparison.}
As shown in Tab.~\ref{tab:preprocessing_cost},
DiffATS preprocessing takes approximately 40 min in total,
compared with 4.5 h for autoencoder pretraining.
Under the reported hardware configurations, autoencoder pretraining
therefore takes approximately $6.8\times$ as long as DiffATS preprocessing.
This ratio compares measured wall-clock times on the respective
CPU and GPU configurations.
In addition to its longer runtime, autoencoder pretraining requires
52.0 GB of peak GPU memory, whereas DiffATS preprocessing
requires no GPU memory.

\noindent\textbf{Anchor-search timing.}
The reported 107 s corresponds to anchor selection using a random
$1{,}000$-trajectory subset of the training set.
The selected anchor is subsequently used for alignment across
the full training set.
Thus, the approximately 40 min total includes anchor search on this
subset and HOOI and OP alignment for all $10{,}000$ trajectories.

\noindent\textbf{Linear-time anchor search.}
The medoid criterion of Sec.~\ref{subsec:mgp} scores each candidate
$i \in [N]$ by $\sum_{j=1}^N \|\Ub_i^\top \Ub_j\|_F^2$, where
$\Ub_i \in \RR^{n \times r}$ denotes the factor matrix of sample $i$
along the mode under consideration.
Evaluating this score for every candidate by enumerating all pairs
$(i, j)$ costs $\cO(N^2 n r^2)$.
Our implementation avoids the pairwise enumeration.
Since $\|\Ub_i^\top \Ub_j\|_F^2 = \mathrm{Tr}(\Ub_i^\top \Ub_j \Ub_j^\top \Ub_i)$,
the score of candidate $i$ can be written as
\begin{equation*}
\sum_{j=1}^N \|\Ub_i^\top \Ub_j\|_F^2
= \mathrm{Tr}\big(\Ub_i^\top \mathbf{S} \Ub_i\big),
\qquad
\mathbf{S} = \sum_{j=1}^N \Ub_j \Ub_j^\top .
\end{equation*}
We therefore form $\mathbf{S} \in \RR^{n \times n}$ once with a single pass
over the candidates, at a cost of $\cO(N n^2 r)$, and then evaluate
$\mathrm{Tr}(\Ub_i^\top \mathbf{S} \Ub_i)$ for every $i$, at a cost of
$\cO(N n^2 r)$.
The candidate with the largest score is the medoid anchor.
For fixed ambient dimension $n$ and rank $r$, the total cost is linear
in the number of samples $N$ rather than quadratic.
Both passes consist of independent per-sample matrix products, which we
parallelize across CPU cores.
The 107 s reported in Tab.~\ref{tab:preprocessing_cost} corresponds to
this implementation.

\noindent\textbf{Scale of the pretraining overhead.}
For reference, diffusion-model training on 2-d Burgers takes
approximately 4 h on a single H100 GPU.
The 4.5 h required for autoencoder pretraining is therefore
comparable to, and slightly exceeds, the diffusion-model training time
on this dataset.
This additional training stage represents a substantial upfront cost
for constructing the autoencoder-based representation.

\subsection{Comparison with an Autoencoder-Based Latent Diffusion Model}
\label{app:ae_latent_diffusion}

\noindent\textbf{Experimental setup.}
The baselines in Tab.~\ref{tab:pde_main} are autoencoder-free.
As a reference, we additionally compare DiffATS with a latent diffusion
model that operates on the latent space of a pretrained autoencoder.
We train a variational autoencoder with the 21.3M-parameter 3-d
convolutional network of App.~\ref{app:ae_comparison} on 2-d Burgers
and then train a conditional diffusion model on its latent codes.
All other settings, including the training and validation split,
the learning rate, and the learning-rate schedule, are identical to
those used for DiffATS.

\noindent\textbf{Metrics.}
In addition to the relative $\ell_2$ error of Tab.~\ref{tab:pde_main},
which is the primary measure of predictive accuracy, we report the
boundary-consistency residuals $r_x$ and $r_y$ and the final-frame error
$\varepsilon_{200}$ defined in App.~\ref{app:pde_physics}.
All metrics are averaged over the $500$ test trajectories and reported
as mean $\pm$ standard deviation over random seeds.

\begin{table}[h]
\centering
\small
\setlength{\tabcolsep}{5pt}
\caption{
\textbf{Comparison with an autoencoder-based latent diffusion model on 2-d Burgers.}
Relative $\ell_2$ error, boundary-consistency residuals $r_x$ and $r_y$,
and final-frame error $\varepsilon_{200}$, reported as mean $\pm$
standard deviation over random seeds. Lower is better for all metrics.
The autoencoder is the 21.3M-parameter 3-d convolutional network of
Tab.~\ref{tab:preprocessing_cost}.
}
\label{tab:ae_latent_diffusion}
\resizebox{\textwidth}{!}{%
\begin{tabular}{@{}lcccc@{}}
\toprule
Method & Rel.\ $\ell_2$ error $\downarrow$ & $r_x \downarrow$ & $r_y \downarrow$ & $\varepsilon_{200} \downarrow$ \\
\midrule
Autoencoder + latent diffusion
& $0.1842 \pm 0.0012$ & $0.01494 \pm 0.00003$ & $0.01530 \pm 0.00005$ & $0.2706 \pm 0.0017$ \\
DiffATS (ours)
& $\mathbf{0.1830 \pm 0.0005}$ & $\mathbf{0.00566 \pm 0.00002}$ & $\mathbf{0.00636 \pm 0.00003}$ & $\mathbf{0.2645 \pm 0.0007}$ \\
\bottomrule
\end{tabular}%
}
\end{table}

\noindent\textbf{Results.}
Tab.~\ref{tab:ae_latent_diffusion} shows that DiffATS performs slightly
better than the autoencoder-based latent diffusion model on all four
metrics.
The relative $\ell_2$ error decreases from $0.1842$ to $0.1830$,
a reduction of approximately $0.7\%$.
The boundary residuals $r_x$ and $r_y$ of DiffATS are less than half of
those of the latent diffusion model, and its final-frame error
$\varepsilon_{200}$ is also lower, which suggests better consistency
with the periodic boundary conditions and better stability at the end
of the prediction horizon.
DiffATS attains this accuracy without the 4.5 h of autoencoder
pretraining on an H100 GPU reported in Tab.~\ref{tab:preprocessing_cost}.
Larger autoencoders and larger training sets may improve the
autoencoder-based model, and we leave this direction to future work.

\subsection{CelebA-HQ}
We additionally report the 
precision and recall~\cite{kynkaanniemi2019improved} as secondary metrics. 
We achieve the highest precision with our method across all baselines and ablation variations,
while the recall of our method is slightly lower than that of AvgPool.

Generated samples from each method are shown in Fig.~\ref{fig:celeba_samples}.
The qualitative comparison is consistent with the quantitative results in Tab.~\ref{tab:celeba}
and Tab.~\ref{tab:celeba_precrec}:
DiffATS produces sharper facial details and more coherent compositions than DCTdiff, SDIFT,
and AvgPool, matching its leading FID and precision scores.

\begin{table}[h]
\centering
\caption{\textbf{Precision and recall on CelebA-HQ at $\mathbf{1024\times1024}$.}
Computed using 10{,}000 generated and 10{,}000 real images.
Boldface indicates the best value in each column among all reported methods.}
\label{tab:celeba_precrec}
\begin{tabular}{l cc}
\toprule
Method & Precision $\uparrow$ & Recall $\uparrow$ \\
\midrule
DCTdiff~\cite{ning2024dctdiff}              & $0.28$              & $0.08$              \\
SDIFT~\cite{chen2025generating}             & $0.00$              & $0.00$              \\
AvgPool                                     & $0.47$          & $\mathbf{0.34}$          \\
DiffATS (ours)                              & $\mathbf{0.66}$ & $0.25$          \\
\midrule
\multicolumn{3}{l}{\emph{Ablations of DiffATS}} \\
Shared bases                  & $0.51$ & $0.32$ \\
Data-dep. bases w/ aug.       & $0.00$ & $0.00$          \\
Data-dep. bases w/o align.    & $0.37$ & $0.03$          \\
\bottomrule
\end{tabular}
\end{table}

\begin{figure}[t]
    \centering
    \includegraphics[width=0.99\textwidth]{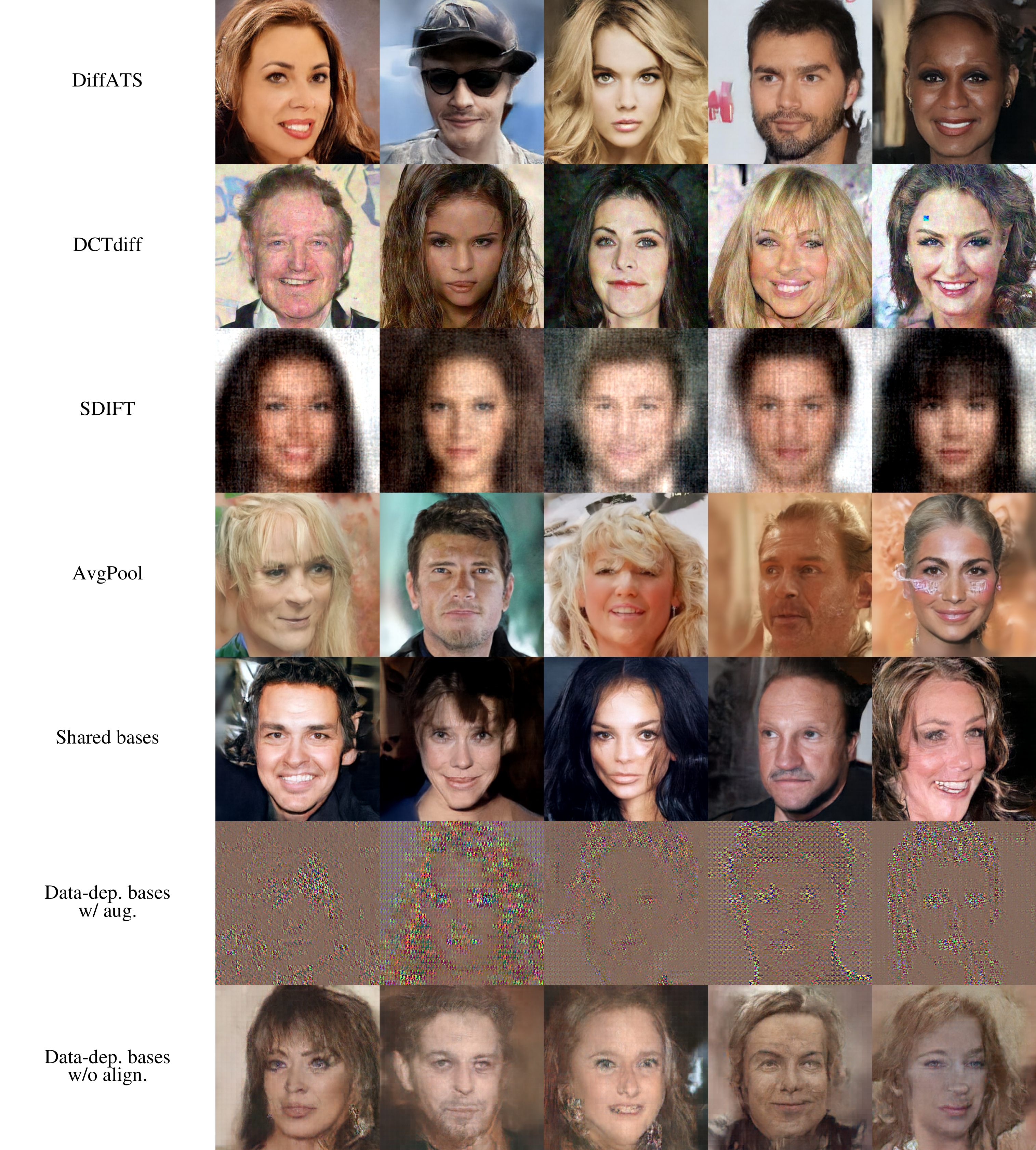}
    \caption{\textbf{Generated samples on CelebA-HQ at $\mathbf{1024 \times 1024}$.} Five samples per method.}
    \label{fig:celeba_samples}
\end{figure}

\subsection{Moving MNIST}
\begin{figure}[t]
    \centering
    \includegraphics[width=0.99\textwidth]{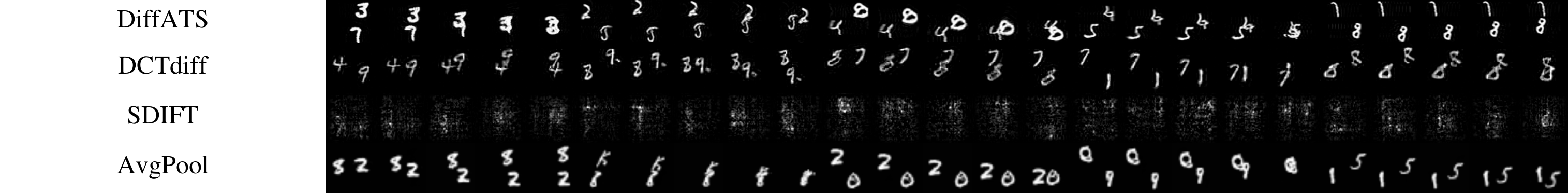}
    \caption{\textbf{Generated samples on Moving MNIST.} Five video samples per method, with each sample shown as a sequence of 5 frames.}
    \label{fig:moving_mnist_samples}
\end{figure}
Fig.~\ref{fig:moving_mnist_samples} shows five generated samples from each method on Moving MNIST, with each sample visualized as 5 frames.
DiffATS produces the sharpest and most coherent digit trajectories, consistent with
its leading FVD score in Tab.~\ref{tab:mmnist_main}.

\subsection{1-d Burgers}
\begin{figure}[H]
    \centering
    \includegraphics[width=0.55\textwidth]{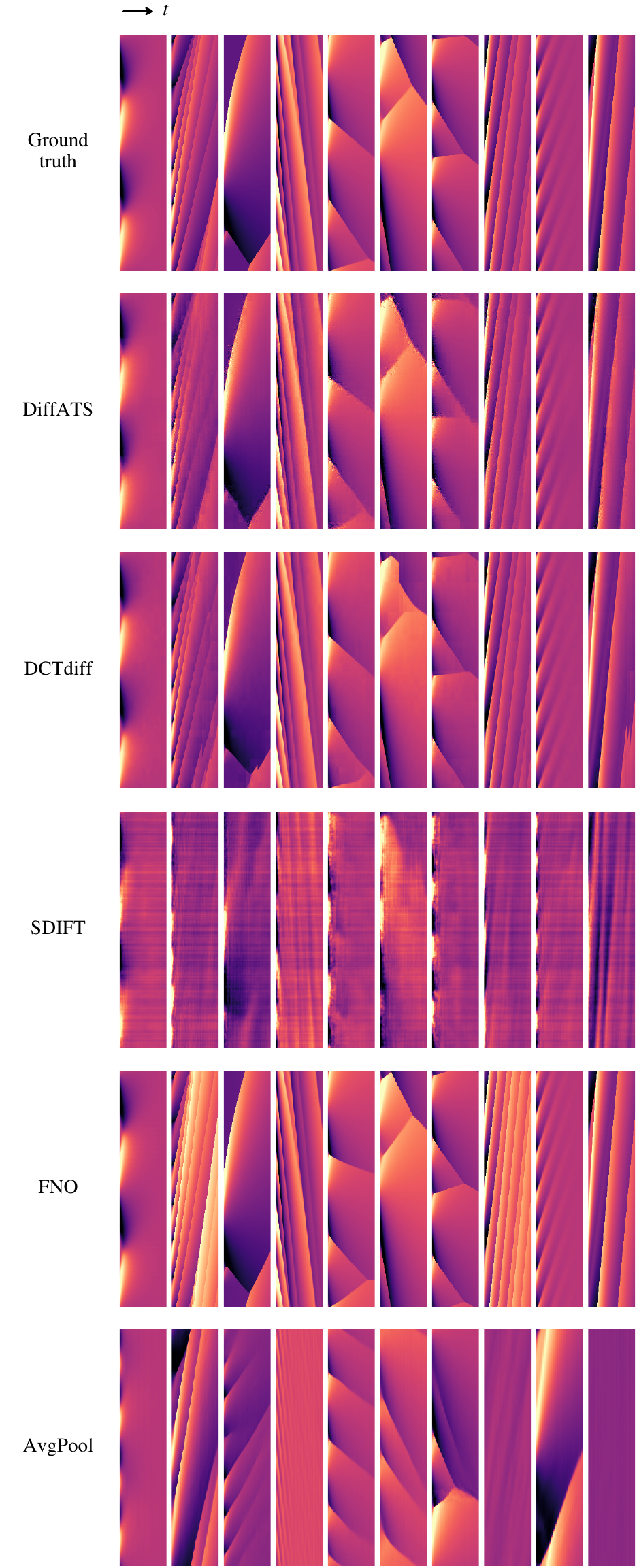}
    \caption{\textbf{Generated samples for 1-d Burgers.}
Conditioned on the initial frame, each model generates the subsequent
trajectory.
Each column corresponds to one of ten randomly selected samples,
visualized as a 2-d space--time plot.
Within each panel, the horizontal and vertical axes represent time
and space, respectively, with time increasing from left to right.
The top row shows the ground truth, while the remaining rows show
the corresponding outputs of DiffATS and the baselines.
Closer agreement with the ground truth indicates better predictive performance.}
    \label{fig:1d_burgers_samples}
\end{figure}

\subsection{1-d reaction-diffusion}
\begin{figure}[H]
    \centering
    \includegraphics[width=0.55\textwidth]{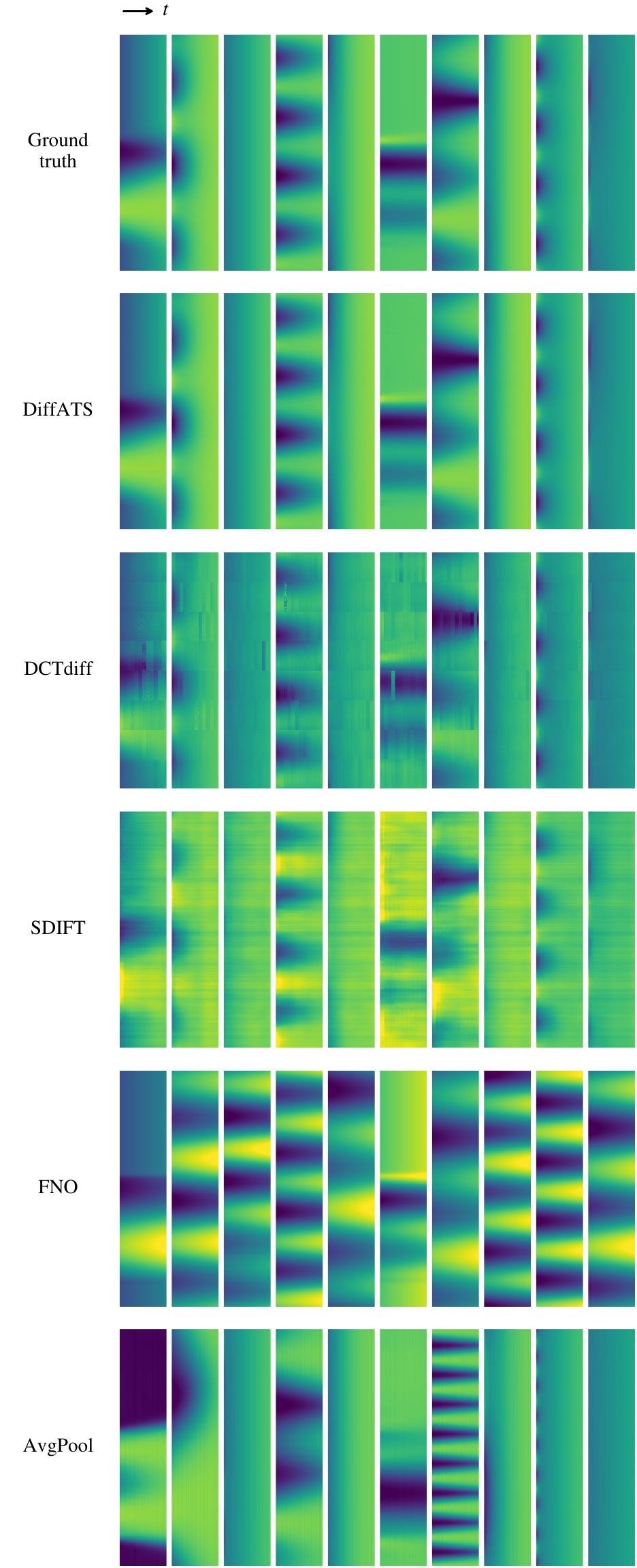}
    \caption{\textbf{Generated samples for 1-d reaction-diffusion.}
Conditioned on the initial frame, each model generates the subsequent
trajectory.
Each column corresponds to one of ten randomly selected samples,
visualized as a 2-d space--time plot.
Within each panel, the horizontal and vertical axes represent time
and space, respectively, with time increasing from left to right.
The top row shows the ground truth, while the remaining rows show
the corresponding outputs of DiffATS and the baselines.
Closer agreement with the ground truth indicates better predictive performance.}
    \label{fig:1d_diffusion_reaction_samples}
\end{figure}

\subsection{2-d Burgers}
\begin{figure}[H]
    \centering
    \includegraphics[width=0.8\textwidth]{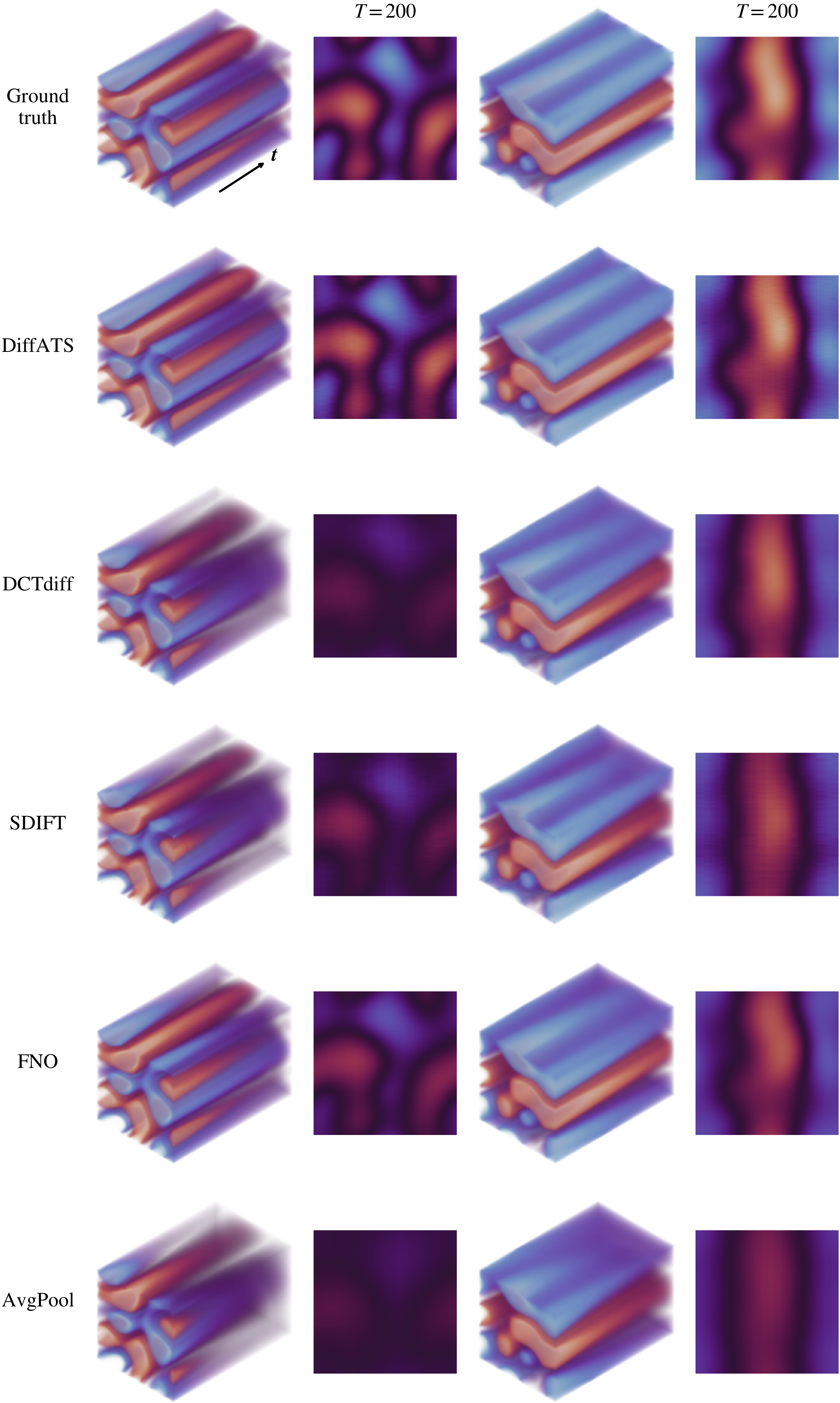}
    \caption{\textbf{Generated samples for 2-d Burgers.}
Conditioned on the initial frame, each model generates all subsequent
frames ($T=1,\ldots,200$).
The left and right pairs of columns show two randomly selected samples
per method. Within each pair, a 3-d spatiotemporal rendering is shown
alongside the corresponding final frame at $T=200$.
The top row shows the ground truth, while the remaining rows show
the outputs of DiffATS and the baselines.
DiffATS most closely matches the spatial structure and temporal evolution
of the ground truth, consistent with the relative $\ell_2$ prediction
error and rMSE results reported in Tab.~\ref{tab:pde_2d_full}.}
    \label{fig:2d_burgers_samples}
\end{figure}

\subsection{2-d K\'arm\'an Vortex Street}
\begin{figure}[H]
    \centering
    \includegraphics[width=1.0\textwidth]{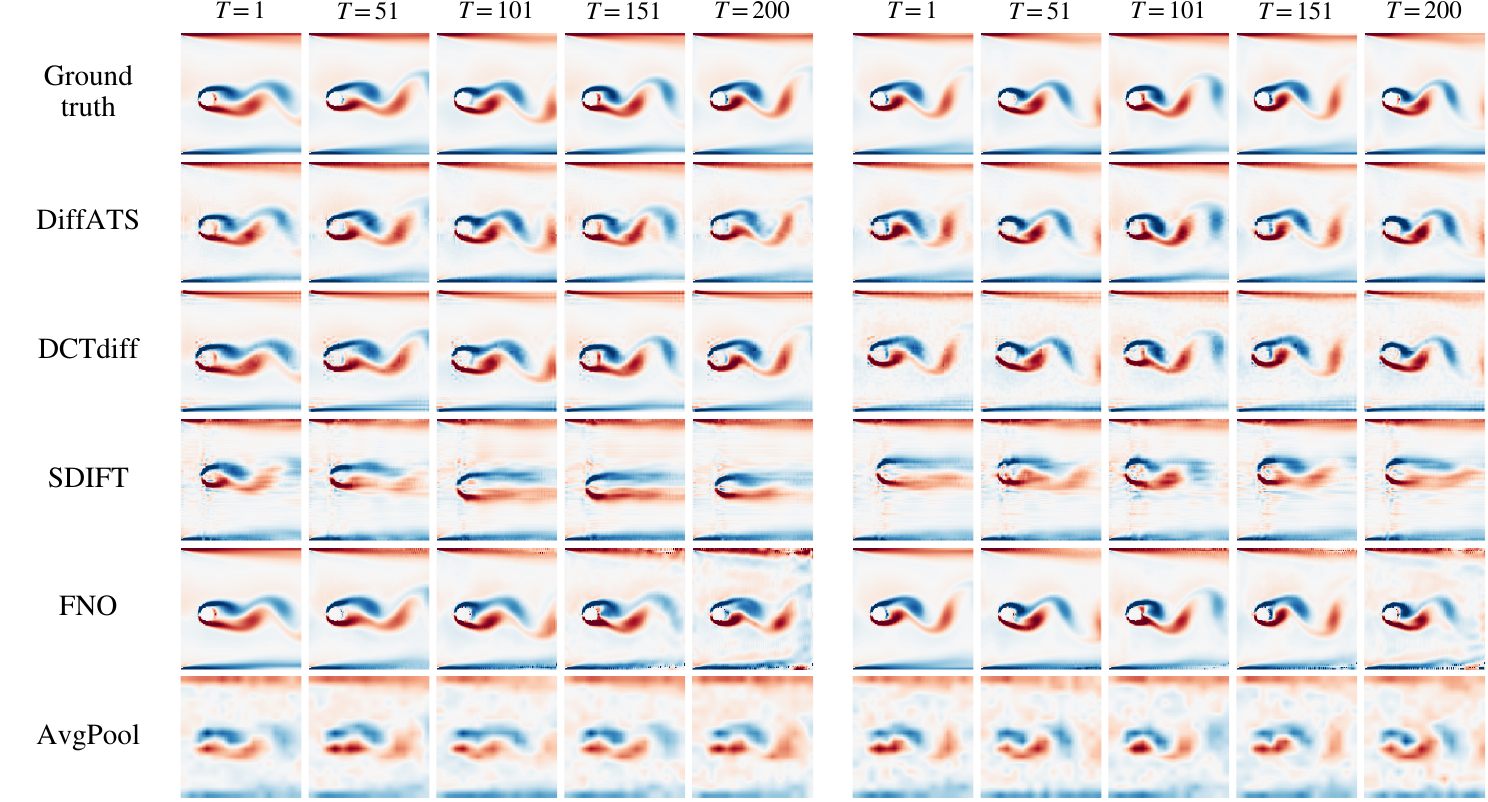}
    \caption{\textbf{Generated samples for the 2-d K\'arm\'an vortex street.}
Conditioned on the initial frame, each model generates all subsequent
frames ($T=1,\ldots,200$).
Two randomly selected samples are shown per method: the left five columns
correspond to one sample, and the right five columns correspond to the other.
The top row shows the ground truth, while the remaining rows show
the outputs of DiffATS and the baselines.
Closer agreement with the ground truth indicates better predictive performance.}
    \label{fig:2d_karman_samples}
\end{figure}



\end{document}